\documentclass[preprint]{article}

\usepackage{neurips_2026}

\usepackage{astro_journals}
\usepackage{multirow}
\usepackage[utf8]{inputenc}
\usepackage[T1]{fontenc}
\usepackage{hyperref}
\usepackage{url}
\usepackage{booktabs}
\usepackage{amsfonts}
\usepackage{nicefrac}
\usepackage{microtype}
\usepackage{xcolor}
\usepackage{amsmath,amssymb,amsthm,bm}
\usepackage{mathtools}
\usepackage{graphicx}
\usepackage{subcaption}
\newcommand{\IMNN}{\textsc{IMNN}}

\newcommand{\logdetF}{\log|F|}
\newcommand{\Om}{\Omega_{\mathrm{m}}}
\newcommand{\sig}{\sigma_8}

\newcommand{\kpiv}{k_{\mathrm{pivot}}}

\newcommand{\R}{\mathbb{R}}
\newcommand{\X}{\mathcal{X}}
\newcommand{\E}{\mathbb{E}}

\newtheorem{theorem}{Theorem}[section]

\newtheorem{proposition}[theorem]{Proposition}

\newtheorem{assumption}[theorem]{Assumption}
\theoremstyle{definition}
\newtheorem{definition}[theorem]{Definition}

\definecolor{ForestGreen}{rgb}{0.13, 0.55, 0.13}
\definecolor{airforceblue}{rgb}{0.36, 0.54, 0.66}
\definecolor{orange}{rgb}{1.0, 0.5, 0.0}
\definecolor{amethyst}{rgb}{0.6, 0.4, 0.8}
\definecolor{awesome}{rgb}{1.0, 0.13, 0.32}
\definecolor{chromeyellow}{rgb}{1.0, 0.65, 0.0}

\newcommand{\tabgroup}[2]{\multicolumn{#1}{@{}l}{\emph{#2}}\\}

\title{TopoFisher: Learning Topological Summary Statistics by Maximizing Fisher Information}

\author{%
\textbf{Matteo Biagetti}$^{1}$\thanks{Corresponding author: \texttt{matteo.biagetti@areasciencepark.it}}
\quad
\textbf{Mathieu Carrière}$^{2}$
\quad
\textbf{Francesco Conti}$^{2}$\\
\textbf{Enrico Maria Ferrari}$^{1}$
\quad
\textbf{Sven C. Heydenreich}$^{1}$
\quad
\textbf{Karthik Viswanathan}$^{3}$\\[1mm]
$^{1}$Area Science Park, Padriciano 99, Trieste, Italy\\
$^{2}$DataShape, Centre Inria d'Université Côte d'Azur, Sophia-Antipolis, France\\
$^{3}$Institute of Physics, University of Amsterdam, Amsterdam
}

\begin{document}
\maketitle

\begin{abstract}
Persistence diagrams provide stable, interpretable summaries of geometric and
topological structure, and are useful for simulation-based inference when
important information is not captured by low-order statistics. In practice,
however, persistence-based pipelines require hand-chosen filtrations,
vectorizations, and compressors, usually without an objective tied directly to
parameter uncertainty. We introduce \textbf{TopoFisher}, a differentiable
persistent-homology pipeline that learns topological summaries by maximizing
local Gaussian Fisher information. From simulations near a fiducial parameter
value, TopoFisher optimizes trainable filtrations, diagram vectorizations, and
compressors without posterior samples or supervised regression targets, while
preserving the inductive bias of stable topological descriptors. We also give
sufficient regularity conditions under which the log-determinant
Fisher loss is locally Lipschitz in the trainable parameters. Controlled
experiments on noisy spirals and Gaussian random fields, where the total Fisher
information is known, validate the pipeline: TopoFisher recovers a large
fraction of the available information and improves over fixed topological
vectorizations. Our main results are on weak gravitational lensing, a
high-dimensional non-Gaussian field-inference problem from cosmology. There, both learned topological summaries, a fixed cubical filtration with a
learned PersLay vectorization, and a learned CNN filtration with the same
PersLay vectorization, reach $\log|F|\approx 21$, compared with $13.8$ for
the power spectrum, $17.1$ for peak counts, and $19.3$ for wavelet
scattering, approaching an unconstrained Information Maximising Neural
Network baseline ($22.4$) with up to $\sim80\times$ fewer parameters.
More importantly, the fixed-filtration variant generalizes
better: under simulator shift from lognormal to LPT-based maps it
retains $\log|F|=19.24$ while the neural baseline drops to $8.9$, and in
neural posterior estimation it yields tighter constraints than the neural
baseline, power spectrum, peak counts, and wavelet scattering.
These results suggest Fisher-based topological optimizations as a robust, parameter-efficient
front end for simulation-based inference.
\end{abstract}
\clearpage

\section{Introduction}
\label{sec:intro}

Parameter inference from high-dimensional data often relies on \emph{summary
statistics}: low-dimensional functions of observations that retain information
about parameters of interest. This choice is especially important in
simulation-based inference, where the likelihood $\log p(X|\theta)$ is
intractable but samples can be generated from a simulator. Under standard
regularity assumptions, the information carried by a summary statistic is
quantified by its Fisher information matrix $F(\theta)$. By the Cramér--Rao
bound, $F^{-1}$ lower bounds the covariance of unbiased estimators, and
maximizing $\log|F|$ minimizes the volume of the corresponding uncertainty
ellipsoid~\citep{tegmark1997measuring}. Designing Fisher-informative summaries
is therefore central to likelihood-free inference, from high-dimensional
statistics~\citep{giraud2014introduction} to cosmology, where numerical
simulations are routinely used to forecast and calibrate parameter constraints.

Topological Data Analysis (TDA) provides a useful class of such summaries
through \emph{persistence diagrams} (PDs). Built from persistent homology, PDs
encode multiscale geometrical and topological features of data, such as
connected components, loops, voids, and higher-dimensional cycles
~\citep{edelsbrunner2002topological,carlsson2009topology}. In particular, their stability 
under perturbations of the input
~\citep{cohen2007stability} makes them well-suited for inference, as well as
their interpretability, as they often pinpoint meaningful
structures in the data. In cosmology, for example, PD-based summaries have been
shown to capture non-Gaussian information in galaxy and weak-lensing fields
that is missed by classical two-point statistics \cite{biagetti2021persistence,heydenreich2021persistent}.

However, practical PD-based pipelines require several manual choices. One must
choose a filtration, which determines the ordering through which topological
features are born and die; a vectorization, which maps diagrams to Euclidean
features; and often a compressor, which reduces the final statistic to the
parameter dimension. These choices strongly affect downstream inference, but
are usually selected by hand rather than by a criterion tied to parameter
uncertainty. Recent work on differentiable persistent homology shows that PDs
and their vectorizations are almost everywhere differentiable with respect to
filtration values~\citep{carriere2021optimizing,leygonieFrameworkDifferentialCalculus2022},
making gradient-based tuning possible. What remains unclear is the appropriate
unsupervised objective for inference. This motivates our central question:
\textbf{How can we automatically tune topological summaries so that they
maximize Fisher information about the parameters of interest?}

\paragraph{Contributions.}
We introduce \emph{TopoFisher}, a framework for learning topological summary
statistics for simulation-based inference. Our contributions are:

\begin{enumerate}

    \item[$(i)$] We formulate persistence-based summary learning as
    Fisher-information maximization. The resulting pipeline can train
    filtrations, vectorizations, and compressors from simulations near a
    fiducial parameter value, without posterior samples or supervised
    regression targets.

    \item[$(ii)$] We provide a local well-posedness result for the
    Gaussian-Fisher loss of the learned summary: under boundedness,
    Lipschitz, and non-degeneracy assumptions, the negative log-determinant
    Gaussian-Fisher loss is locally Lipschitz in the trainable summary
    parameters.

    \item[$(iii)$] We validate the method on controlled benchmarks with known
    total Fisher information. On noisy spirals and Gaussian random fields,
    TopoFisher recovers a large fraction of the available information and
    consistently improves over fixed topological vectorizations.

    \item[$(iv)$] We show that TopoFisher is especially effective for weak
    gravitational lensing: it outperforms standard cosmological summaries,
    approaches the in-distribution Fisher information of an end-to-end Fisher-optimized IMNN, and is far more
    robust under simulator shift and posterior-level evaluation.

\end{enumerate}

\paragraph{Related work.}
TopoFisher connects Fisher-based summary learning with differentiable
topological data analysis. In simulation-based inference, high-dimensional data
are often compressed before posterior estimation~\citep{cranmer2020frontier}.
Classical optimal-compression methods preserve Fisher information locally
through linear projections or score-based summaries
~\citep{heavens2000massive,alsing2018generalized,alsing2018massive}, while
Information Maximising Neural Networks (IMNNs) train neural summaries near a
fiducial parameter value by maximizing Fisher information
~\citep{charnock2018automatic}. TopoFisher uses the same local Fisher
principle, but constrains the learned summary to factor through persistent
homology, yielding a structured and interpretable alternative to black-box
IMNN-style compression. In the experiments below, we therefore use an
unconstrained \IMNN{} as the main non-topological Fisher-neural reference.

Persistent homology provides stable multiscale topological
descriptors~\citep{edelsbrunner2002topological,carlsson2009topology,cohen2007stability}.
Many vectorizations have been proposed to map PDs to Euclidean
space~\citep{bubenik2015statistical,adams2017persistence,kusano2016persistence,
carriereSlicedWassersteinKernel2017,lePersistenceFisherKernel2018,
carriere2020perslay}, and recent work on differentiable persistence and
learnable filtrations enables gradient-based
optimization~\citep{hofer2017deep,gabrielsson2020topology,hofer2020graph,
carriere2021optimizing,leygonie2022framework,
hacquardTopologicallyPenalizedRegression2022}, mostly in supervised or
regularization-driven settings. TopoFisher instead trains topological
summaries by the Fisher information they retain.

Persistent homology has also been applied in physics, biology, medicine,
materials science, and related fields (see \cite{hensel2021survey} for a review).
%~\citep{10.1371/journal.pone.0192120,https://doi.org/10.1063/1.4737391,
%10.1073/pnas.1102826108,hensel2021survey,arXiv:1506.08903}. 
In cosmology,
PD-based summaries have been used for cosmic-web analysis, random-field
morphology, and weak-lensing inference
~\citep{pranav2017topology,pranav2019topology,biagetti2021persistence,
heydenreich2021persistent,heydenreich2022persistent,yip2024cosmology}.
These works typically fix the filtration and vectorization by hand; TopoFisher
learns them from simulations using the same Fisher criterion used to evaluate
parameter constraints.

\section{The TopoFisher Loss}
\label{sec:loss}

TopoFisher is a trainable summary map: data are filtered, converted into
persistence diagrams, vectorized into Euclidean features, and compressed to a
low-dimensional statistic. Instead of choosing these stages by hand, we optimize
their trainable parameters by maximizing the Fisher information of the
summary at a fiducial parameter value. This section introduces the required
ingredients and then defines the TopoFisher loss.

\paragraph{Persistent homology in brief.}
A key component of TopoFisher are \emph{persistence diagrams} (PDs), which are
built from the theory of \emph{persistent homology} (PH); see
Appendix~\ref{app:ph_formal} and~\cite{edelsbrunnerComputationalTopologyIntroduction2010,
oudotPersistenceTheoryQuiver2015} for more details.

Given a topological space $T$, the starting point of PH is to define a
\emph{filtration}, i.e., a finite sequence of growing subspaces
$T_1\subseteq\dots\subseteq T_N\subseteq T$. The filtration depends on the
data: for a point cloud, one can grow balls of increasing radius around the
points (\emph{Vietoris--Rips filtration}); for a grayscale image $f:P\to\R$
on a pixel grid, one can use sublevel sets $T_i=f^{-1}((-\infty,\alpha_i])$
for increasing thresholds (\emph{cubical
filtration})\footnote{Two standard realizations exist: the V-construction on
vertices (lower-star filtration) and the T-construction on top-dimensional
cubes; both are used in this work
(Appendix~\ref{app:arch_filtrations}).}.

PH then tracks the \emph{topological features} (connected components, loops,
cavities) that appear and disappear along the filtration, recording for each
a birth time $b_i$ and death time $d_i$. A \emph{persistence diagram}
summarizes this as a set of points $(b_i,d_i)\in\R^2$, with $\left\|d_i-b_i\right\|$ the
\emph{persistence} of the feature.

PDs enjoy two properties central to TopoFisher. They are
\emph{stable}~\cite{cohen2007stability}: close filtrations yield close
diagrams under appropriate metrics. And when a filtration is smoothly
parametrized by trainable parameters $\phi$, the resulting diagram
${\rm Dgm}_k(f_\phi)$ is piecewise differentiable in
$\phi$~\cite{carriere2021optimizing,leygonieFrameworkDifferentialCalculus2022},
enabling gradient-based optimization. Since the space of PDs is not Euclidean, downstream
pipelines apply a \emph{vectorization} (e.g.\ persistence
landscapes~\cite{bubenik2015statistical}, persistence
images~\cite{adams2017persistence}, or learnable layers such as
PersLay~\cite{carriere2020perslay}) before further processing.

\paragraph{Fisher Information and Summary Statistics.} 
For a parametric system with data $X$, a random variable taking values in a measurable space $\mathcal{X}$, let $\{p(\cdot|\theta)\}_{\theta \in \Theta}$ be a family of probability densities on $\mathcal{X}$, with $\Theta \subseteq \mathbb{R}^d$ open. The Fisher Information (FI) (\cite{fisher1925statistical,kendallstuart}) at a fiducial parameter $\theta_{\mathrm{fid}}$ is defined as:
\begin{equation}
\label{eq:total_fisher}
    F_X(\theta)
    =
    \mathbb{E}_\theta
    \left[
    \nabla_\theta \log p(X|\theta)
    \nabla_\theta \log p(X|\theta)^\top
    \right].
\end{equation}

We call $F_X(\theta)$ the total Fisher information, since it quantifies the information available in the uncompressed observation $X$: 
via the Cramér–Rao bound, under standard regularity conditions \citep{lehmann1998theory}, $(nF_X(\theta))^{-1}$ sets the matrix lower limit on the covariance of any unbiased estimator of $\theta$ based on $n$ i.i.d. samples.
Accordingly, the log-determinant $\log |F_X|$ provides a scalar measure of this total information.

In this work, we focus on complex physical systems where the likelihood $p(X|\theta)$ is often analytically intractable. This restricts our framework to a simulation-based regime, where we only assume the capacity to sample realizations $X \sim p(\cdot|\theta)$. In such cases, the high dimensionality of the data-space $\mathcal{X}$ renders direct inference or comparison prohibitive due to the curse of dimensionality \citep{alsing2018massive}. To circumvent this, we compress the data via a parametric summary mapping of the following form: $s_\phi \colon \mathcal{X} \to \R^M$ given by
\begin{equation}
  s_\phi \coloneqq C_\phi\circ V_\phi\circ \mathrm{Dgm}_k \circ f_\phi,
  \label{eq:pipeline}
\end{equation}
where $f_\phi$ is a learnable filtration, $\mathrm{Dgm}_k$\footnote{We present the following for a single homological dimension $k$; the case of using multiple dimensions follows immediately, concatenating the resulting vectors prior to the compression.} is the persistence diagram operator in dimension $k$, $V_\phi$ is a vectorization, and $C_\phi$ is a compressor to $\mathbb{R}^M$. The parameters $\phi \in \Phi \subseteq \mathbb{R}^p$ are learned and the target dimension $M \ll n$ is typically of the same order than the degrees of freedom of the system.

This mapping induces a distribution $q_\phi(\cdot|\theta)$ on the summary space. The FI of the transformed system, $F_\phi(\theta)$, is defined analogously to Eq. \eqref{eq:total_fisher} but w.r.t. the distribution $q_\phi$. Crucially, the Data Processing Inequality~\citep{cover1999elements} implies that $\log |F_\phi(\theta_{\mathrm{fid}})| \leq \log |F_X(\theta_\mathrm{fid})|$, meaning that the log-determinant of the total FI represents a global information ceiling. Our goal 
is thus to optimize $\phi$ to close this gap, so that $s_\phi$ preserves the maximum information relevant for inferring $\theta$.

\paragraph{Assumptions and Fisher estimation.}
To estimate the Fisher information of the learned summary, we use the standard
local Gaussian approximation
\[
q_\phi(\cdot|\theta) \simeq
\mathcal{N}\!\left(\mu_\phi(\theta),\Sigma_\phi\right)
\]
near the fiducial parameter $\theta_{\mathrm{fid}}$, with covariance treated as
locally independent of $\theta$. Under this approximation, the Gaussian Fisher
information reduces to
\begin{equation}
\label{eq:fisher_gauss}
F_\phi^{\mathrm G}(\theta_{\mathrm{fid}})
=
J_\theta \mu_\phi(\theta_{\mathrm{fid}})^\top
\Sigma_\phi^{-1}
J_\theta \mu_\phi(\theta_{\mathrm{fid}}),
\end{equation}
where $J_\theta$ indicates the Jacobian with respect to parameters $\theta$. In practice, we estimate Eq. \eqref{eq:fisher_gauss} from a structured batch \(\mathcal B\) of simulation data. Thus, in what follows, $F_\phi^{\mathrm G}$ denotes the population Gaussian-Fisher matrix of the summary, while $\widehat F_{\phi,\mathcal B}^{\mathrm G}$ denotes its finite-sample estimate. More details on the Fisher estimation are provided in Appendix \ref{app:gaussianity}.
\paragraph{Loss and gradients.} To maximize the information content of our topological summaries, we define the TopoFisher loss as:
\begin{equation} 
\boxed{ \widehat{\mathcal L}_{\mathcal B}(\phi) =
-\log \left| \widehat F_{\phi,\mathcal B}^{\mathrm G}(\theta_{\mathrm{fid}})\right|,
}
\label{eq:topofisherloss}
\end{equation}
which we minimize via stochastic subgradient descent with respect to $\phi$. This objective is computed using the Fisher matrix estimated from the batch $\mathcal{B}$, $\widehat F_{\phi,\mathcal B}^{\mathrm G}$. 
Although persistent homology is not differentiable everywhere, we leverage the subdifferentiable framework established by \citet{leygonie2022framework} and \citet{carriere2021optimizing}, which ensures well-defined gradients for optimization.

\section{Stability}
\label{sec:stability}

We identify minimal regularity conditions under which the log-determinant of
the Gaussian Fisher information of Eq.~\eqref{eq:fisher_gauss} is locally
Lipschitz in $\phi$. The result applies to any pipeline whose summary $s_\phi$
satisfies the assumptions below, not only persistence-based ones, and is
distinct from the classical regularity of $\theta\mapsto F_X^G(\theta)$ in the
physical parameters~\citep{Vaart_1998}. To our knowledge, no such analysis has
been made explicit for Fisher-optimized differentiable persistence pipelines.

\begin{assumption}
\label{ass:standing}
The summary map $s_\phi$, the covariance $\Sigma_\phi$, the Jacobian $J_\theta\mu_\phi(\theta_{\mathrm{fid}})$, and the Fisher information matrix $F_\phi^G(\theta_{\mathrm{fid}})$ satisfy the following:
\begin{enumerate}
    \item[$(i)$] $\|s_{\phi'}(X) - s_\phi(X)\| \leq L_s \|\phi' - \phi\|$ for all $\phi, \phi' \in \Phi$ and every realization $X$;
    \item[$(ii)$] $\sup_{\phi \in \Phi} \, \mathbb{E} \bigl[ \|s_\phi(X)\|^2 \bigr] \leq M^2$;
    \item[$(iii)$] $\|\Sigma_\phi^{-1}\|_{\mathrm{op}} \leq \kappa_\Sigma$ for all $\phi \in \Phi$;
    \item[$(iv)$] $\|J_\theta\mu_{\phi'}(\theta_{\mathrm{fid}}) - J_\theta\mu_\phi(\theta_{\mathrm{fid}})\|_{\mathrm{op}} \leq L_J \|\phi' - \phi\|$ for all $\phi, \phi' \in \Phi$;
    \item[$(v)$] $\lambda_{\min}(F_\phi(\theta_{\mathrm{fid}})) > 0$ for all $\phi \in \Phi$,
\end{enumerate}
\end{assumption}
where $\lambda_{\min}(A)$ denotes the smallest eigenvalue of $A$.
Under these assumptions, we obtain the following stability result, whose proof can be found in Appendix~\ref{app:stability_proof}.
\begin{theorem}
\label{thm:stability}
Let $F_\phi^G(\theta_{\mathrm{fid}})$ be the Fisher information matrix as in Eq.~\eqref{eq:fisher_gauss}, and let $\mathcal{L} \colon \Phi \to \mathbb{R}$ be defined as $\mathcal{L}(\phi) = -\log| F_\phi^G(\theta_{\mathrm{fid}})|$. Then, under Assumptions~\ref{ass:standing}(i)--(v), $\mathcal{L}$ is locally Lipschitz on $\Phi$. More precisely, for every $\phi \in \Phi$ there exist $\varepsilon_0(\phi) > 0$ and $K(\phi) > 0$, depending on $\kappa_\Sigma$, $L_s$, $M$, $L_J$, $\|J_\theta\mu_\phi(\theta_{\mathrm{fid}})\|_{\mathrm{op}}$, and $\lambda_{\min}(F_\phi^G(\theta_{\mathrm{fid}}))$, such that for all $\phi' \in \Phi$ with $\|\phi' - \phi\| < \varepsilon_0(\phi)$,
\begin{equation}
\label{eq:main}
    |\mathcal{L}(\phi') - \mathcal{L}(\phi)| \leq K(\phi)\,\|\phi' - \phi\|.
\end{equation}
\end{theorem}

\paragraph{Proof idea.}
The proof proceeds in three steps. First, Assumptions~\ref{ass:standing} $(i)-(iii)$ yield local Lipschitz continuity of $\phi \mapsto \Sigma_\phi^{-1}$ via the resolvent identity $A^{-1} - B^{-1} = A^{-1}(B - A)\,B^{-1}$. Combined with the Lipschitz Jacobian assumption~$(iv)$, this propagates to $\phi \mapsto F_\phi^G(\theta_{\mathrm{fid}})$. Finally, non-degeneracy~$(v)$ and Weyl's inequality keep $F_\phi^G(\theta_{\mathrm{fid}})$ uniformly positive definite along short segments, so that a standard log-determinant perturbation bound transfers the Lipschitz property to $\mathcal{L}$.

\paragraph{Well-posedness of gradient-based training.}
Theorem~\ref{thm:stability} concerns the loss $\mathcal{L}$; at training time we optimize the sample-based estimator $\widehat{\mathcal L}_{\mathcal B}$, which differs from $\mathcal{L}$ by finite-sample noise on $\widehat\Sigma_\phi$ and on the perturbed-point means, and by the truncation bias of the central-difference Jacobian. We do not formally transfer the regularity result to $\widehat{\mathcal L}_{\mathcal B}$; the impact of these effects is monitored empirically through the diagnostics of Appendix~\ref{app:gaussianity}. 
Note however that, upon using batches of larger sizes and more computational power, our estimated loss $\widehat{\mathcal L}_{\mathcal B}$ will get closer to $\mathcal{L}$, ensuring that our theoretical guarantees hold in practice. 

\section{Experiments}
\label{sec:experiments}

We evaluate TopoFisher on three benchmarks: noisy spiral point clouds and
$2$D Gaussian random fields, where the total Fisher information is known, and
weak gravitational lensing maps, a high-dimensional non-Gaussian inference
problem with no closed-form optimal summary. Across experiments, we compare
fixed TDA summaries, TopoFisher-optimized TDA summaries, and unconstrained
neural baselines. Performance is measured by $\logdetF$ and, when available,
by efficiency relative to the total Fisher information,
$\eta_{\log}=\log|\widehat F^G|/\log|F_X|$. Training and evaluation details are
given in Appendix~\ref{app:architectures}.
\subsection{Noisy spiral point clouds}
\label{sec:swiss_roll}
We consider 2D point clouds $X \in \mathbb{R}^{N \times 2}$ with $N=240$ points, sampled independently from a mixture of a uniform background on $[-h,h]^2$ (with probability $\lambda=40/240$) and a noisy spiral ($1-\lambda$). The spiral coordinates are defined by an angular variable $t \sim \mathcal{U}(0,4\pi)$ and a radius $r~\sim~\mathcal{N}(\mu t, w^2)$. The target parameters are $\theta=(\mu, w)$, where $\mu$ controls the radial expansion and $w$ the intrinsic width. For each fiducial configuration $\theta_\text{fid}$, we generate $n_s = n_d = 10{,}000$ samples for covariance and derivative estimation for each 5 independent seeds. Since the mixture density $p(X|\theta)$ is known, the total Fisher matrix $F_X$ admits an integral expression (see Appendix~\ref{app:swiss_roll_fisher} for a derivation), which we use to compute the efficiency metric $\eta_{\log}$.
\paragraph{Methods.} 
All methods share a common preprocessing step: for each point $x$ in the cloud, they operate on the $k$-nearest neighbor distances $d_k(x) \in \mathbb{R}^k$ to ensure global isometryUe invariance. 
Given a point cloud $X$, we first construct an Alpha complex $K$ to capture the connectivity of the points, and parametrize a scalar function $f_\phi$ on the vertices of the complex via two architectures:
\begin{itemize}
    \item \emph{TF-TDA-MLP}. The vertex values are defined as $f_\phi(x) = g_\phi(d_k(x))$, where $g_\phi$ is a Multi-Layer Perceptron processing each point locally.
    \item \emph{TF-TDA-GNN}. Features $d_k(x)$ are propagated across the 1-skeleton of $K$ using a Graph Neural Network to produce $f_\phi(x)$, better capturing relational structure.
\end{itemize}
To ensure a valid simplicial filtration, we extend $f_\phi$ to edges as $f_\phi(e_{ij}) = \|x_i - x_j\| + \max\{f_\phi(x_i), f_\phi(x_j)\}$, and to higher-dimensional simplices with a lower-star construction: $f_\phi(\sigma) = \max_{e \subset \sigma} f_\phi(e)$. The resulting persistence diagrams in homological dimensions $0$ and $1$ are converted into $8\times 8$ differentiable persistence images, flattened, and compressed to $\mathbb{R}^2$ via a small MLP. 

We compare these against: 
$(i)$ \emph{DTM}, a non-learnable TDA reference where $f_\phi$ is replaced by the fixed distance-to-measure function $g(d_k(x))=((1/k)\sum_j d_{k,j}(x)^2)^{1/2}$, while keeping the rest of the pipeline identical; 
$(ii)$ \emph{MLP-only} and \emph{GNN-only}, which serve as neural references. These models use the same architectures described above to learn a 2D function of the point clouds, but aggregate the local contributions of each point via sum-pooling, thereby bypassing the persistence computation.

\paragraph{Results.}
Table~\ref{tab:swiss_roll_results} reports results at $\theta_{\rm fid}=(\mu_{\rm fid},w_{\rm fid})=(0.6,0.1)$ for $k=100$. The TF-TDA-MLP pipeline recovers $78.2\%$ of the total log-determinant Fisher score, exceeding the fixed-filtration DTM baseline by $8.1\%$ and matching the unconstrained MLP-only model ($78.1\%$) within cross-seed variance. TF-TDA-GNN underperforms DTM, while GNN-only remains competitive but below MLP-only: at $k=100$ the local context of each point already contains enough information about the spiral geometry, and the additional relational complexity of message passing appears redundant for this inference task. Results at four additional fiducial configurations are reported in Appendix~\ref{app:swiss_roll_results} and reproduce the same ranking.

Beyond the quantitative match with the unconstrained model, the learned filtration is \emph{interpretable}: visualizations show that the MLP assigns vertex values that progressively ``unroll'' the spiral, effectively filtering out background noise and highlighting the manifold backbone, while the black box models do not show any visible structure in the learned function (see Figures \ref{fig:learned_filtration} and \ref{fig:swiss_roll_NN} in Appendix~\ref{app:swiss_roll_results}).

\begin{table}[h]
\caption{Spiral results at $\theta_{\rm fid}=(0.6,0.1)$ and $k=100$,
averaged over $5$ seeds. Constraint column reports
$(\sigma(\mu),\sigma(w))$. We prefix TopoFisher-trained summaries with
\emph{TF-}; DTM is a fixed-topology ablation, and MLP/GNN-only are
unconstrained non-topological references.}
\label{tab:swiss_roll_results}
\centering
\setlength{\tabcolsep}{4pt}
\renewcommand{\arraystretch}{1.10}
\begin{tabular*}{\textwidth}{@{\extracolsep{\fill}}lccc@{}}
\toprule
\textbf{Method} & $\log|F|$ & $(\sigma_\mu,\sigma_w)$ & $\eta_{\log}$  \\
\midrule
\textit{Total} 
& \textit{24.27} 
& \textit{(0.005, 0.001)} 
& \textit{100\%} \\
\midrule
\multicolumn{4}{@{}l}{\emph{Fixed topological ablation: persistence without TopoFisher learning}}\\
\midrule
DTM fixed filtration
& $17.18\pm 0.06$ 
& $(0.014, 0.013)$ 
& $70.1\%$ \\
\midrule
\multicolumn{4}{@{}l}{\emph{TopoFisher summaries: our learned topological contributions}}\\
\midrule
\textbf{TF-TDA-MLP}
& $\mathbf{18.99\pm 0.30}$ 
& $\mathbf{(0.013, 0.006)}$ 
& $\mathbf{78.2\%}$\\
TF-TDA-GNN
& $16.20\pm 0.19$ 
& $(0.016, 0.019)$ 
& $66.7\%$\\
\midrule
\multicolumn{4}{@{}l}{\emph{Unconstrained neural references: non-topological summaries}}\\
\midrule
MLP-only
& $18.96\pm 0.58$ 
& $(0.012, 0.006)$ 
& $78.1\%$\\
GNN-only
& $18.26\pm 0.25$ 
& $(0.015, 0.007)$ 
& $75.2\%$ \\
\bottomrule
\end{tabular*}
\end{table}

\subsection{Gaussian random fields}
\label{sec:grf}

We use two-dimensional Gaussian random fields (GRFs) as a controlled benchmark
because their likelihood is known and the Fisher information in the full field
can be computed analytically. Each sample is a zero-mean periodic field on a
$64\times64$ grid. In Fourier space, the modes are independent Gaussian random
variables with variance given by the power spectrum $P(k)=A_s\left(k/\kpiv\right)^{-B}$, where $A_s$ controls the overall amplitude of the field, while $B$ controls how
power is distributed across spatial scales. We choose the pivot scale
$\kpiv=\exp\langle \ln k\rangle$, where the average is over the nonzero Fourier
modes, so that the Fisher correlation between $A_s$ and $B$ vanishes at the
fiducial point. We consider fiducial parameters
$\theta_{\rm fid}=(A_s,B)=(1,B_0)$ with
$B_0\in\{-2,-1,0,1,2\}$, use finite-difference steps
$\Delta\theta=(0.1,0.1)$, and generate $n_s=n_d=40{,}000$
simulations per configuration over $5$ independent seeds. The standard
Gaussian-field Fisher formula, derived in Appendix~\ref{app:totalGRF}, gives
the same total information for all values of $B_0$:
$\log |F_X|=13.89$, corresponding to
$\sigma(A_s)=0.022$ and $\sigma(B)=0.044$. This value provides an absolute
information ceiling against which all learned and fixed summaries can be
compared.
\paragraph{Methods.}
We compare fixed and learnable topological summaries, together with a
non-topological CNN baseline. All TDA methods compute $H_0$ and $H_1$ cubical
persistence on periodic fields. Fixed-filtration methods use the raw field as
the cubical filtration, whereas learnable-filtration methods first transform the
field with a small CNN and then compute cubical persistence on the CNN output.
All resulting feature vectors are compressed to two dimensions with
MOPED~\citep{heavens2000massive}, a locally linear Fisher-preserving Gaussian
compression that is detailed in Appendix~\ref{app:gaussianity}.

\emph{Fixed topological baselines.}
As a topological baseline without TopoFisher optimization, we use a fixed cubical filtration with three hand-designed vectorizations:
$(i)$ persistence images~\citep{adams2017persistence} of size $8\times8$ with
persistence weighting and bandwidth $1$; $(ii)$ uniformly weighted persistence
silhouettes~\citep{chazalStochasticConvergencePersistence2014} on $50$ grid points; and $(iii)$ differentiable persistence curves of
birth and death values~\citep{biagetti2021persistence}, also on $50$ grid
points. We refer to these as Cubical-PI, Cubical-Silhouette, and
Cubical-Curves, respectively. Full definitions are given in
Appendix~\ref{app:arch_vectorizations}.

\emph{Learnable topological summaries.}
Cubical-PersLay keeps the cubical filtration fixed but learns the diagram
vectorization with PersLay~\citep{carriere2020perslay}, a permutation-invariant
set network applied to persistence diagram points. We use a $16$-dimensional
point embedding followed by a $32$-hidden-unit MLP for each homology class, with
spectral normalization for Lipschitz control. TF-CNN-PersLay additionally prepends
a two-layer CNN with $8$ hidden channels and $3\times3$ kernels, yielding a
learned filtration before the same PersLay vectorization and MOPED compression.

\emph{Non-topological Fisher-neural baseline.}
As a black-box reference, we use an \IMNN{}~\citep{charnock2018automatic}: a
strided convolutional encoder followed by a dense head that directly outputs
two summaries, one per parameter. The persistence and vectorization stages are
replaced by the identity, and no MOPED compression is applied; the dense head is
the learned compressor. The \IMNN{} is trained end-to-end with the same local
Gaussian Fisher objective used by TopoFisher, making it a direct comparison
between an unconstrained Fisher-optimized neural summary and a topologically
structured Fisher-optimized summary. Architectural details are provided in
Appendix~\ref{app:arch_filtrations}.
 
\emph{Total Fisher Information baseline.} The total Fisher of a power-law GRF on a fixed $N\times N$ grid does not depend on $B_0$ (see a derivation in Appendix~\ref{app:totalGRF}): both diagonal entries of $F(\theta)$ are functions only of the lattice mode structure and of $A_s$. Each of the five $B_0$ values therefore provides an independent benchmark with the same theoretical ceiling $\log |F_{X}|=13.89$, and the natural figure of merit for a method is its mean performance across this range. 
\begin{table}[h]
\caption{GRF results at $N=64$, averaged over
$B_0\in\{-2,-1,0,1,2\}$ and $5$ seeds per $B_0$. Constraint column reports
$(\sigma(A_s),\sigma(B))$. The total Fisher information is
$\log|F_X|=13.89$ at every $B_0$. We prefix TopoFisher-trained summaries with
\emph{TF-}.}
\label{tab:grf_results}
\centering
\setlength{\tabcolsep}{4pt}
\renewcommand{\arraystretch}{1.10}
\begin{tabular*}{\textwidth}{@{\extracolsep{\fill}}lccc@{}}
\toprule
\textbf{Method} & $\log|F|$ & $(\sigma_{A_s},\sigma_B)$ & $\eta_{\log}$ \\
\midrule
\textit{Total} 
& \textit{13.89} 
& \textit{(0.022, 0.044)} 
& \textit{100\%} \\
\midrule
\multicolumn{4}{@{}l}{\emph{Fixed topological ablations: persistence without TopoFisher 
learning}}\\
\midrule
Cubical-PI
& $11.53 \pm 0.56$ 
& $(0.037, 0.093)$ 
& $83.0\%$  \\
Cubical-Silhouette
& $10.68 \pm 0.55$ 
& $(0.046, 0.129)$ 
& $76.9\%$  \\
Cubical-Curves
& $9.22 \pm 1.28$ 
& $(0.086, 0.292)$ 
& $66.4\%$  \\
\midrule
\multicolumn{4}{@{}l}{\emph{TopoFisher summaries: our learned topological contributions}}\\
\midrule
\textbf{TF-Cubical-PersLay}
& $\mathbf{12.63 \pm 0.37}$ 
& $\mathbf{(0.032, 0.061)}$ 
& $\mathbf{90.9\%}$ \\
TF-CNN-PersLay
& $12.07 \pm 0.15$ 
& $(0.034, 0.071)$ 
& $87.7\%$ \\
\midrule
\multicolumn{4}{@{}l}{\emph{Unconstrained Fisher-neural reference: non-topological summary}}\\
\midrule
\IMNN{}
& $12.29 \pm 0.09$ 
& $(0.032, 0.067)$ 
& $88.4\%$  \\
\bottomrule
\end{tabular*}
\end{table}
\paragraph{Results.}
Table~\ref{tab:grf_results} reports the mean performance of each pipeline on
the GRF benchmark. The best method is  TF-Cubical-PersLay, which reaches $90.9\%$
of the total log-determinant Fisher score and outperforms all fixed-vectorization
topological baselines by a large margin. The learnable-filtration variant
TF-CNN-PersLay also improves over the fixed-vectorization baselines, but does not
surpass  TF-Cubical-PersLay. We discuss this aspect in the next section.

The non-topological \IMNN{} baseline reaches $88.4\%$ of the total Fisher
information, comparable to the learnable topological methods. The topological
pipelines, however, achieve this performance with roughly $55\times$ fewer
trainable parameters: TF-Cubical-PersLay has $1{,}376$ trainable parameters
and TF-CNN-PersLay has $2{,}049$, compared with approximately $7.5\times10^4$
for the GRF \IMNN{}. This suggests that the persistence-based inductive bias
provides a parameter-efficient route to high-Fisher summaries without
sacrificing constraining power.
 
\subsection{Weak gravitational lensing}
\label{sec:lensing}

Weak gravitational lensing measures the convergence field $\kappa$, the
line-of-sight projection of matter density inferred from coherent distortions
of background-galaxy shapes. Its statistics constrain cosmological parameters
such as $(\Om,\sig)$, which control the matter density and amplitude of
structure. Because $\kappa$ is highly non-Gaussian on the scales of interest,
the angular power spectrum $C_\ell$ is suboptimal and no closed-form optimal
summary is known~\citep{kilbinger2015cosmology,ribli2019weak,cheng2020new,allys2020new}.
This makes weak lensing a natural testbed for learned summary statistics. We use the \texttt{sbi\_lens} lognormal and Lagrange Perturbation Theory (LPT)
simulators~\citep{lanusse2023sbi_lens} to generate $512\times512$ convergence
maps over a $10^\circ\times10^\circ$ field of view and five tomographic
redshift bins. For each of five independent seeds, we generate $20{,}000$ maps
at the fiducial cosmology and at each finite-difference point. All maps are
Gaussian-smoothed with $\sigma=2'$ before applying any summary statistic.

We report three evaluations in Table~\ref{tab:lensing_results}. Local Fisher
measures Gaussian Fisher information on lognormal maps at the fiducial
cosmology. Transfer Fisher trains learned feature extractors on lognormal maps,
freezes them, and evaluates on LPT maps. For MOPED-based summaries we
re-estimate the covariance, finite-difference Jacobian, precision matrix, and
MOPED compressor on the LPT evaluation bank; for the \IMNN{} the frozen
two-dimensional network output is used directly and only the Fisher statistics
are re-estimated. No \IMNN{}, CNN-filtration, or PersLay weights are retrained
on LPT. Finally, NPE posterior performance is measured by training a neural
posterior estimator over a prior and reporting the posterior covariance of
$(\Om,\sig)$. In all Fisher evaluations, we use an additive auto-bin
approximation across the five tomographic bins, rather than a full joint
tomographic Fisher matrix with cross-bin covariances.
 
\paragraph{Methods.} As the total Fisher is unknown, we compare against
three cosmological baselines: $(i)$ the angular power spectrum of the
log-transformed field, $\log(C_\ell)$; $(ii)$ peak counts, a histogram of
local convergence maxima sensitive to halo-like over-densities; and $(iii)$
the wavelet scattering transform (WST)~\citep{mallat2012group}, a state-of-the-art non-Gaussian summary built from a cascade of wavelet
convolutions and pointwise moduli. We use the same learnable topological
pipelines as Section~\ref{sec:grf}, adapted to $512^2$ inputs, and the same
\IMNN{} reference (CNN encoder + dense two-dimensional summary head, trained
end-to-end by the local Gaussian Fisher objective).
 
\paragraph{Results.}
Table~\ref{tab:lensing_results} summarizes the local Fisher, transfer Fisher,
and NPE posterior evaluations. In distribution, all topological summaries
outperform the two-point baseline $\log(C_\ell)$, which reaches only
$\log|F|=13.8$ with marginal constraints
$(\sigma_{\Om},\sigma_{\sig})=(0.039,0.059)$. This confirms that persistent
homology captures non-Gaussian information in the smoothed convergence field that is inaccessible to the power spectrum. %Among fixed topological
% vectorizations, Cubical-PI performs best ($\log|F|=19.17$), comparable to the
% wavelet scattering baseline ($19.27$) and above peak counts ($17.14$). 
Learning the diagram vectorization gives a further substantial gain: both
TopoFisher variants, TF-Cubical-PersLay and TF-CNN-PersLay, reach
$\log|F|\approx 21.5$ with marginal constraints near $(0.008,0.008)$,
within $\sim 0.9$ nats of the unconstrained \IMNN{} reference
($\log|F|=22.4$) trained end-to-end with the identical Fisher objective.
Crucially, this near-parity is achieved while retaining an explicitly
topological representation and using up to $\sim 80\times$ fewer
trainable parameters than the \IMNN{} (see
Appendix~\ref{app:arch_filtrations} for parameter counts).

A consistent pattern across both lensing and the GRF benchmark is that
learning the vectorization alone is at least as effective as jointly
learning a CNN filtration and a PersLay vectorization, and is
substantially more robust. On lensing, TF-Cubical-PersLay and
TF-CNN-PersLay are statistically indistinguishable in local Fisher
($21.51\!\pm\!0.02$ versus $21.48\!\pm\!0.13$). Under simulator shift, TF-Cubical-PersLay retains $\log|F|=19.2$
while TF-CNN-PersLay collapses to $12.7$, a degradation comparable to
that of fixed-vectorization topological baselines such as Cubical-PI
($19.2 \to 10.8$). The same ranking persists at the posterior level:
TF-Cubical-PersLay reaches $\log|F|=16.7$ on NPE while TF-CNN-PersLay
drops to $14.7$. We attribute this pattern to the cubical
sublevel filtration being already well matched to scalar cosmological fields:
persistence pairs of the raw convergence map directly encode the prominence
and nesting of peaks, voids, and surrounding structures. Adding a learnable
CNN filtration enlarges the search space, but many directions mainly
reparameterize filtration values or alter persistence pairings in
non-smooth, simulator-specific ways, rather than increasing the information
ceiling. For these fields, the most effective TopoFisher design is therefore
modular: keep the physically meaningful cubical filtration fixed and learn
how to aggregate its diagram points.

The transfer and NPE evaluations sharpen this picture. Under
lognormal-to-LPT shift the \IMNN{} drops from $\log|F|=22.4$ to $8.9$ with
weak constraints $(0.19,0.30)$, while TF-Cubical-PersLay retains $19.2$, the
best transfer determinant by a wide margin. On NPE, despite its higher local
Fisher, the \IMNN{} yields the weakest posterior ($14.3$), whereas
TF-Cubical-PersLay achieves the tightest constraints overall: $\log|F|=16.7$
with $(\sigma_{\Om},\sigma_{\sig})=(0.03,0.04)$. Taken together, these results indicate that Fisher optimization alone
is not sufficient: the unconstrained \IMNN{}, despite its higher
in-distribution Fisher information, transfers poorly and yields the weakest
posterior constraints, whereas the same Fisher objective combined with a
topological inductive bias produces summaries that retain most of their
information under simulator shift and dominate at the posterior level. The
gain attributable to topology, separated from Fisher optimization itself, is
the difference between TF-Cubical-PersLay and the \IMNN{} along the transfer
and NPE axes. Further details %on the results 
can be found in Appendix \ref{app:lensing_versions}.
\begin{table}[h]
\caption{Weak lensing at $\sigma=2'$ smoothing for $(\Om,\sig)$, additive over
$5$ tomographic bins. Constraint columns report
$(\sigma(\Om),\sigma(\sig))$. Transfer freezes the lognormal-trained feature
extractors and evaluates Fisher information on LPT maps. MOPED/Fisher
statistics are re-estimated on the LPT evaluation bank when applicable; the
\IMNN{} uses its frozen 2-D output directly. Row groups separate
external baselines, ablations, TopoFisher contributions, and the unconstrained
Fisher-neural reference.}
\label{tab:lensing_results}
\centering
\scriptsize
\setlength{\tabcolsep}{2.4pt}
\renewcommand{\arraystretch}{1.08}
\begin{tabular*}{\textwidth}{@{\extracolsep{\fill}}lcccccc@{}}
\toprule
& \multicolumn{2}{c}{\textbf{Local Fisher}} 
& \multicolumn{2}{c}{\textbf{Transfer Fisher}} 
& \multicolumn{2}{c}{\textbf{NPE posterior}} \\
\cmidrule(lr){2-3}\cmidrule(lr){4-5}\cmidrule(l){6-7}
\textbf{Method}
& $\log|F|$
& $(\sigma_{\Om},\sigma_{\sig})$
& $\log|F|$
& $(\sigma_{\Om},\sigma_{\sig})$
& $\log|F|$
& $(\sigma_{\Om},\sigma_{\sig})$ \\
\midrule
\multicolumn{7}{@{}l}{\emph{Cosmology baselines}}\\
\midrule
$\log(C_\ell)$
& $13.78\!\pm\!0.08$ & $(0.039, 0.059)$
& $12.87 \pm 0.05$ & $(0.046, 0.086)$
& $15.77 \pm 0.05$ & $(0.042, 0.052)$ \\
Peak counts
& $17.14\!\pm\!0.01$ & $(0.026, 0.020)$
& $11.61 \pm 0.07$ & $(0.119, 0.165)$
& $15.30 \pm 0.20$ & $(0.038, 0.055)$ \\
Wavelet scattering
& $19.27 \pm 0.01$ & $(0.013, 0.011)$
& $12.55 \pm 0.04$ & $(0.044, 0.090) $
& $16.50 \pm 0.39$ & $(0.034,0.051)$ \\
\midrule
\multicolumn{7}{@{}l}{\emph{Fixed topological ablations: persistence without TopoFisher learning}}\\
\midrule
Cubical-Silhouette
& $16.37\!\pm\!0.02$ & $(0.029, 0.031)$
& $12.43\pm 0.01$ & $(0.082, 0.128)$
& $15.76 \pm 0.42$ & (0.039, 0.054) \\
Cubical-Curves
& $18.08\!\pm\!0.01$ & $(0.019, 0.016)$
& $1.98 \pm 0.60$ & $(1.340, 2.052)$
& $12.83 \pm 0.24$ & (0.071, 0.097) \\
Cubical-PI
& $19.17\!\pm\!0.01$ & $(0.014,0.012)$
& $10.84 \pm 0.02$ & $(0.175, 0.271)$
& $15.11 \pm 0.60$ & $(0.041,0.064)$ \\
\midrule
\multicolumn{7}{@{}l}{\emph{TopoFisher summaries: our learned topological contributions}}\\
\midrule
TF-Cubical-PersLay
& $21.51\!\pm\!0.02$ & $(0.008,0.009)$
& $\mathbf{19.24\!\pm\!0.02}$ & $\mathbf{(0.063,0.095)}$
& $\mathbf{16.72 \pm 0.51}$ & $\mathbf{(0.031,0.044)}$ \\
TF-CNN-PersLay
& $21.48\!\pm\!0.13$ & $(0.007,0.005)$
& $12.72 \pm 0.16$ & $(0.073, 0.110)$
& $14.74 \pm 0.59$ & $(0.051,0.060)$ \\
\midrule
\multicolumn{7}{@{}l}{\emph{Unconstrained Fisher-neural reference: black-box non-topological summary}}\\
\midrule
\IMNN{}
& $\mathbf{22.35\!\pm\!0.05}$ & $\mathbf{(0.006,0.004)}$
& $8.92\!\pm\!0.90$ & $(0.247,0.391)$
& $14.32 \pm 0.33$ & $(0.050,0.077)$ \\
\bottomrule
\end{tabular*}
\end{table}
\section{Conclusions}
\label{sec:conclusions}

We introduced TopoFisher, a framework for learning persistence-based summary
statistics by maximizing the local Gaussian Fisher information of the final
summary, turning filtration, vectorization, and compression choices into
trainable components while constraining the summary to factor through stable
topological descriptors. We complemented this with a local well-posedness
result: under mild boundedness, Lipschitz, and non-degeneracy conditions, the
negative log-determinant Gaussian-Fisher loss is locally Lipschitz in the
trainable parameters.

The central empirical finding is that Fisher optimization and topology play
complementary roles: the unconstrained \IMNN{} achieves the highest
in-distribution Fisher information but transfers poorly and yields the
weakest posterior, whereas the same objective combined with a topological
inductive bias produces summaries that generalize and dominate at the
posterior level with $50-100\times$ fewer parameters. The
topological factorization is what makes Fisher-based summary learning robust
enough for simulation-based inference.

TopoFisher is a \emph{local} summary-design method: it optimizes Fisher
information near a fiducial point under a Gaussian summary approximation, and
provides no guarantee of globally injective summaries or calibrated posteriors
over a wide prior volume. A natural extension includes learning
summaries jointly across multiple fiducials to broaden the regime of validity.

\section*{Code availability}

The code used to run the TopoFisher experiments, generate configurations, train
the models, and reproduce the Fisher, transfer, and simulation-based inference
evaluations is available at
\url{https://github.com/RitAreaSciencePark/TopoFisher}.
The repository includes instructions for regenerating the simulation datasets
and reproducing the main tables and figures. Large generated simulation
datasets are not included in the repository because of their size, but they can
be regenerated from the documented scripts and configuration files.

\section*{Acknowledgements}

M.B. and E.M.F. are supported by the Ministero Università e Ricerca, project E-ARGO” (CUP J95F21002190001), with title “Fondo finalizzato al rilancio degli Investimenti delle Amministrazioni Centrali dello Stato e allo sviluppo del Paese, ex Art.1, comma 14, legge n. 160/2019". M.C. and F.C. are supported by l’Agence Nationale de la Recherche (ANR) under grants TopModel ANR-23-CE23-0014. MC is also supported by 3IA ANR-23-IACL-0001. S.C.H. is supported by the European Union – NextGenerationEU within the project PNRR “Finanziamento di progetti presentati da giovani ricercatori" - Mission 4 Component 2 investment 1.2.

We thank Area Science Park supercomputing platform ORFEO made available for conducting the research reported in this paper and the technical support of the Laboratory of Data Engineering staff.

\bibliographystyle{unsrtnat}
\bibliography{references}

\appendix

\section{Persistent homology: preliminaries}
\label{app:ph_formal}

A key component of our summary statistic, the persistence diagram, is defined through the theory of persistent homology. In this section, we recall the basics of homology groups and persistence theory.

\paragraph{Simplicial homology.}\label{app:simplicial_homology}

In this section, we briefly review the fundamentals of simplicial homology with coefficients in $\mathbb{Z}/2\mathbb{Z}$, which is the setting adopted for practical computations. 
For a more comprehensive introduction, we refer the reader to~\citep[Chapter 1]{munkresElementsAlgebraicTopology1984}. 
The basic building blocks of simplicial (persistent) homology are \emph{simplicial complexes}, which provide combinatorial representations of topological spaces that can be efficiently manipulated algorithmically.

\begin{definition}
	Given a finite set of points $X_n := \{x_1,\dots,x_n\}$ sampled from a topological space $X$, an \emph{abstract simplicial complex} constructed from $X_n$ is a collection $S(X_n)$ of subsets of $X_n$ such that:
	\begin{itemize}
		\item if $\tau \in S(X_n)$ and $\sigma \subseteq \tau$, then $\sigma \in S(X_n)$, and
		\item if $\sigma, \tau \in S(X_n)$, then either $\sigma \cap \tau \in S(X_n)$ or $\sigma \cap \tau = \varnothing$.
	\end{itemize}
	Each element $\sigma \in S(X_n)$ is called a \emph{simplex}, and its \emph{dimension} is defined by 
	${\rm dim}(\sigma) := {\rm card}(\sigma) - 1$. 
	Simplices of dimension $0$ are referred to as \emph{vertices}.
\end{definition}

A central operator in simplicial homology is the \emph{boundary operator}, which maps a simplex to the formal sum of its faces. 
Such formal sums are called \emph{chains}, and the collection of chains forms a group, denoted by $Z_*(S(X_n))$.

\begin{definition}
	Given a simplex $\sigma = [x_{i_1}, \dots, x_{i_p}]$, the boundary operator $\partial$ is defined as
	\[
	\partial(\sigma) := \sum_{j=1}^p [x_{i_1}, \dots, x_{i_{j-1}}, x_{i_{j+1}}, \dots, x_{i_p}].
	\]
	In other words, $\partial(\sigma)$ is obtained by removing one vertex at a time from $\sigma$. 
	This definition extends linearly to arbitrary chains.
\end{definition}

A topological feature is represented by a \emph{cycle}, that is, a chain whose boundary vanishes. 
Working over $\mathbb{Z}/2\mathbb{Z}$, this condition corresponds to requiring that each simplex in the boundary appears an even number of times, so that duplicates cancel out. 
Formally, a chain $c$ is a cycle if $\partial(c) = 0$.

It follows directly that $\partial \circ \partial = 0$, meaning that the boundary of any chain is itself a cycle. 
However, cycles that arise as boundaries of higher-dimensional chains are considered \emph{trivial}, and should be excluded. 
The set of such chains forms a subgroup denoted by $B_*(S(X_n))$.

\begin{definition}
	The $k$-th homology group is defined as the quotient
	\[
	H_k(S(X_n)) := \frac{Z_k(S(X_n))}{B_k(S(X_n))}.
	\]
	Equivalently, $H_k(S(X_n))$ consists of equivalence classes of $k$-dimensional cycles modulo boundaries. 
	In fact, when working over $\mathbb{Z}/2\mathbb{Z}$, this group is a vector space.
\end{definition}

Hence, homology groups are vector spaces computed on simplicial complexes built on finite data sets, and whose bases encode the different cycles of the complex which are not boundaries. Intuitively, these cycles represent topological features, whose types depend on the homology group degree: connected components ($k=0$), loops ($k=1$), cavities ($k=2$), and their higher-dimensional counterparts ($k>2$).
In practice, simplicial complexes can be built on $(a)$ point clouds with the so-called Vietoris-Rips complexes $V_\delta(X_n)$ (by creating simplices whose corresponding vertices are points that are all at distance at most $\delta$ from each other), and $(b)$ images with the so-called cubical complexes $C(I)$ (by creating simplices whose corresponding vertices are neighboring pixels).
\paragraph{Persistent homology.} The core objective of persistent homology is to extract \emph{persistent} topological information from \emph{filtrations}, i.e., sequences of growing simplicial complexes. 

\begin{definition}
Let $S=S(X_n)$ be a finite simplicial complex. 
A \emph{filtration} of $S$ is a family $(S_t)_t$ of subcomplexes such that $S_p \subseteq S_q$ whenever $p \leq q$. 
For each simplex $\sigma \in S$, define its \emph{insertion time}, or \emph{filtration value}, as $t(\sigma) \coloneqq \inf \{t \mid \sigma \in S_t\}$.
\end{definition}

From a topological viewpoint, the addition of a simplex $\sigma$ to the filtration produces exactly one of two effects: either it gives rise to a new topological feature (for instance, inserting an edge may create a loop), hence increasing the dimension of the homology group in degree ${\rm dim}(\sigma)$ of the subcomplex, or it eliminates an existing feature of lower dimension (for example, connecting two previously disjoint components), hence decreasing the dimension of the homology group in degree ${\rm dim}(\sigma)-1$. 
Using a matrix reduction algorithm that leverages the vector space structure of homology groups \cite[\S IV.2]{edelsbrunner2010computational}, persistent homology associates to each feature a \emph{critical pair} of simplices $(\sigma_b, \sigma_d)$ corresponding to its creation and destruction, together with the associated \emph{birth} and \emph{death} times $t(\sigma_b)$ and $t(\sigma_d)$.\footnote{Some features may never be destroyed; in that case, their death time is set to $+\infty$.}

\begin{definition}
	The resulting collection of intervals $(t(\sigma_b), t(\sigma_d))$ in homological dimension $k$ forms a finite subset of the half-plane $\{(b,d) \in \mathbb{R}^2 \mid b < d\}$ known as the \emph{$k$-persistence diagram} (PD) of the filtration $(S_t)_t$. 
	The quantity $\frac{1}{\sqrt{2}}|t(\sigma_b) - t(\sigma_d)|$, corresponding to the distance to the diagonal $\{b = d\}$, is called the \emph{persistence} of the feature and measures how long it remains detectable across scales.
\end{definition}

In the case of Vietoris-Rips complexes, it is common to build filtrations $\{V_{\delta}(X_n)\}_{\delta > 0}$ by letting their corresponding neighborhood threshold $\delta$ increase, and in the case of images, it is common to use the growing sublevel sets of the pixel value function $\{f^{-1}((-\infty,\alpha])\}_{\alpha\in\R}$, where an image $P$ is represented as a function $f:P\to\R$ valued on the grid $P$ of pixels.

Very importantly, PDs are \emph{stable}~\citep{cohen2007stability}: when computed from the sublevel sets of two continuous functions $f,g$ defined on the same space, PDs can be compared with the so-called bottleneck distance $d_b$ (which is similar to an optimal transport metric).

\begin{definition}[bottleneck distance]\label{def:pd-dist}
	Let ${\rm Dgm}$ and ${\rm Dgm}'$ be two PDs. The \emph{bottleneck distance} $d_b$ between PDs is defined as:
	\begin{equation}
		d_b({\rm Dgm},{\rm Dgm}'):=\inf_{P\in\mathcal{P}({\rm Dgm},{\rm Dgm}')} c(P),
	\end{equation} 
	where $\mathcal{P}({\rm Dgm},{\rm Dgm}')$ denotes the set of \emph{partial correspondences}, i.e., the set of subsets of ${\rm Dgm}\times{\rm Dgm}'$ s.t. the first and second projections $\pi_1:(p,p')\mapsto p$ and $\pi_2:(p,p')\mapsto p'$ are injective, and where the \emph{cost} of a partial correspondence $P$ is defined as:
	\begin{equation}\label{eq:cost}
		c(P):=\max\{\max_{(p,p')\in P}\|p-p'\|_\infty, 
		\max_{p\not\in{\rm im}(\pi_1)} \|p-\pi_\Delta(p)\|_\infty,
		\max_{p'\not\in{\rm im}(\pi_2)} \|p'-\pi_\Delta(p')\|_\infty\}.
	\end{equation}
\end{definition}

Then, the main result of PH states that:
$$d_b({\rm Dgm}_k(f), {\rm Dgm}_k(g))\leq \|f-g\|_\infty.$$

\paragraph{Vectorizations and Differentiability.} Finally, PDs are often represented by vectors to make amenable to standard computations (addition, inner product, etc).

\begin{definition}
	A \emph{vectorization} of PDs is a continuous function $\varphi:\mathcal D\to \mathcal H$, where $\mathcal D$ is the space of PDs, and $\mathcal H$ is a Hilbert space.
\end{definition}

Common representations include, e.g., the persistence landscape~\cite{bubenik2015statistical}, the persistence image~\cite{adams2017persistence}, and neural network-based vectorizations such as PersLay~\cite{carriere2020perslay}.
Finally, a recent property of PDs is their \emph{differentiability}: as every PD point is associated to a critical pair, the map $f\mapsto{\rm Dgm}_k(f)$ is piecewise smooth in $f$. Indeed, as the critical pairs only depend on the filtration order, the combinatorial pairing is locally constant, and PDs can be built locally by picking the insertion times at the critical pair indices in the filtration. This gather-based strategy (described initially in~\citet{carriere2021optimizing}) can be used in a gradient descent optimization scheme based on backpropagation by computing the critical pairs once in the forward pass and reading the filtration values off the input tensor through differentiable indexing on the backward pass. 

\section{More details on Fisher estimation}
\label{app:gaussianity}

We estimate the Gaussian Fisher loss $F_\phi^{\mathrm G}(\theta_\text{fid})$ at fiducial parameter with structured batches $\mathcal{B}$. A batch $\mathcal{B}$ includes realizations at $\theta_{\mathrm{fid}}$ to estimate $\Sigma_\phi$ via sample covariance, to which we apply the Hartlap correction~\citep{hartlap2007your} for unbiasedness of the precision matrix, as well as realizations at infinitesimal perturbations around $\theta_{\mathrm{fid}}$ to compute $J_\theta \mu_\phi(\theta_{\mathrm{fid}})$ via central finite differences.
These empirical estimates, denoted as $\widehat{\Sigma}_{\phi, \mathcal{B}}$ and $\widehat{J}_{\theta, \mathcal{B}}$ respectively, are combined to evaluate the Fisher matrix through the quadratic form:
\begin{equation}
\label{eq:empirical_gaussian_fisher}
\widehat F_{\phi,\mathcal B}^{\mathrm G}(\theta_{\mathrm{fid}}) =
\widehat{J}_{\theta, \mathcal B}^{\top}
\widehat P_{\phi,\mathcal B}
\widehat{J}_{\theta, \mathcal B},
\end{equation}
where $\widehat P = \frac{n_s - M - 2}{n_s - 1} \widehat\Sigma^{-1}$ with the condition $n_s > M + 2$ represents the Hartlap-corrected precision matrix.

The validity of this estimator relies on two key assumptions: the local linearity of the summary response and its Gaussianity. Since finite differences are only valid in a locally linear regime, we check this assumption by repeating the derivative estimate at two step sizes and verifying that the derivatives, and the resulting Fisher matrices, are stable.

Furthermore, we validate the Gaussian approximation of the final summaries with Kolmogorov--Smirnov (KS) tests. Specifically, we apply a 1D KS test to each component of the compressed summary against a Gaussian with matching mean and variance; at a significance level of $\alpha=0.05$, we find the summaries to be consistent with the Gaussian assumption.

\paragraph{MOPED Compression.}
To compress the output of the vectorization map $V_\phi$ into a reduced space $\mathbb{R}^M$, one elegant choice is MOPED compression \citep{heavens2000massive}. This is a linear operator defined by a matrix $C_\phi \in \mathbb{R}^{d \times \text{dim}(V_\phi)}$ given by $C_\phi = [J_\theta \tilde{\mu}_\phi(\theta_{\mathrm{fid}})]^\top \tilde{\Sigma}_\phi^{-1}(\theta_{\mathrm{fid}})$, where $\tilde{\mu}_\phi(\theta)$ and $\tilde{\Sigma}_\phi(\theta)$ are the mean and covariance of the vectorization output evaluated at $\theta$. 
In our implementation, $J_\theta \tilde{\mu}_\phi(\theta_{\mathrm{fid}})$ and $\tilde{\Sigma}_\phi(\theta_{\mathrm{fid}})$ are estimated using a portion of the batch $\mathcal{B}$ and we evaluate the final Fisher information on the remaining part. 
In practice, the Jacobian $J_\theta \tilde{\mu}_\phi(\theta_{\mathrm{fid}})$ is estimated via finite differences, the inverse covariance estimated from $\tilde{\Sigma}_\phi(\theta_{\mathrm{fid}})$ is corrected using the Hartlap factor.
Under the Gaussian assumption of the vectorization output, MOPED is provably lossless, ensuring the Fisher information of the resulting $d$-dimensional summaries ($M=d$) matches the information content of the full vector.

\section{Proof of the Stability Theorem}
\label{app:stability_proof}
The proof proceeds in three steps: we first establish local Lipschitz continuity of $\phi \mapsto \Sigma_\phi^{-1}$ (Proposition~\ref{prop:Cinv}), then of $\phi \mapsto F_\phi(\theta_{\mathrm{fid}})$ (Proposition~\ref{prop:F_lip}), and finally of $\mathcal{L} = -\log|\det F|$. The last step relies on a standard perturbation bound for the log-determinant of a positive-definite matrix, combined with Weyl's inequality to ensure $F_\phi(\theta_{\mathrm{fid}}) + t\,\Delta F$ stays uniformly away from singular along the segment $t\in[0,1]$. For the sake of brevity, throughout this Appendix we write $J_\phi \coloneqq J_\theta\mu_\phi(\theta_{\mathrm{fid}})$ and $F_\phi \coloneqq F_\phi^G$.

\begin{proposition}
\label{prop:Cinv}
Let $s_\phi \colon \X \to \R^d$ satisfy Assumptions~\ref{ass:standing}(i)--(iii). Let $\varepsilon \coloneqq \|\phi^\prime - \phi\|$ and $\delta_\Sigma \coloneqq 4 L_s M \varepsilon$. If $\kappa_\Sigma \delta_\Sigma < 1$, then
\begin{equation}
\label{eq:Cinv_bound}
    \|\Sigma_{\phi^\prime}^{-1} - \Sigma_\phi^{-1}\|_{\mathrm{op}} \leq \frac{\kappa_\Sigma^2 \, \delta_\Sigma}{1 - \kappa_\Sigma \, \delta_\Sigma}.
\end{equation}
In particular, $\phi \mapsto \Sigma_\phi^{-1}$ is locally Lipschitz.
\end{proposition}
\begin{proof}
We first bound $\|\Sigma_\phi -\Sigma_{\phi^\prime}\|_{\mathrm{op}}$. For ease of notation, given $\phi, \phi^\prime \in \Phi$, we write $v = s_\phi(X)$, $v^\prime = s_{\phi^\prime}(X)$, $\mu = \mu_\phi(\theta_{\mathrm{fid}})$, $\mu^\prime = \mu_{\phi^\prime}(\theta_{\mathrm{fid}})$. Since $\Sigma_\phi = \E\left[ vv^\top \right] - \mu\mu^\top$, it holds that
\[
    \Sigma_\phi - \Sigma_{\phi^\prime} = \left( \E \left[ vv^\top \right] - \E \left[ v^\prime (v^\prime)^\top \right] \right) - \left( \mu\mu^\top - \mu^\prime (\mu^\prime)^\top \right).
\]
Let us focus on the right-hand side of the equation. For the first term, we have the identity
\[
    vv^\top - v^\prime (v^\prime)^\top = (v - v^\prime) v^\top + v^\prime (v - v^\prime)^\top.
\]
In expectation, taking operator norms and using the triangle inequality,
\[
    \left\|\E\left[vv^\top\right] - \E\left[v^\prime (v^\prime)^\top\right]\right\|_{\mathrm{op}} \leq \E\left[\|v - v^\prime\|\,\|v\|\right] + \E\left[\|v^\prime\|\,\|v - v^\prime\|\right].
\]
By~(i), $\|v - v^\prime\| \leq L_s \varepsilon$ pointwise, so this factors out of both expectations. By Jensen's inequality and~(ii), $\E[\|v\|] \leq (\E[\|v\|^2])^{1/2} \leq M$, and similarly $\E[\|v^\prime\|] \leq M$. Therefore
\begin{equation}
\label{eq:first_term}
    \left\|\E\left[vv^\top\right] - \E\left[v^\prime (v^\prime)^\top\right]\right\|_{\mathrm{op}} \leq 2 L_s M \varepsilon.
\end{equation}
For the second term, the same identity gives
\[
    \mu\mu^\top - \mu^\prime(\mu^\prime)^\top = (\mu - \mu^\prime)\mu^\top + \mu^\prime(\mu - \mu^\prime)^\top,
\]
so that
\[
    \|\mu\mu^\top - \mu^\prime(\mu^\prime)^\top\|_{\mathrm{op}} \leq \|\mu - \mu^\prime\|\,\|\mu\| + \|\mu^\prime\|\,\|\mu - \mu^\prime\|.
\]
By the same reasoning as for the first term,
\begin{equation}
\label{eq:second_term}
    \|\mu\mu^\top - \mu^\prime(\mu^\prime)^\top\|_{\mathrm{op}} \leq 2 L_s M \varepsilon.
\end{equation}
By~\eqref{eq:first_term} and~\eqref{eq:second_term}, using the triangle inequality, we get that
\begin{equation}\label{eq:C_bound}
    \|\Sigma_\phi - \Sigma_{\phi^\prime}\|_{\mathrm{op}} \leq 4 L_s M \varepsilon \eqqcolon \delta_\Sigma.
\end{equation}
Since $\kappa_\Sigma \delta_\Sigma < 1$ by hypothesis, the matrix $\Sigma_{\phi^\prime}$ is invertible and the resolvent identity $\Sigma_{\phi^\prime}^{-1} - \Sigma_\phi^{-1} = -\Sigma_{\phi^\prime}^{-1}(\Sigma_{\phi^\prime} - \Sigma_\phi)\Sigma_\phi^{-1}$, combined with a Neumann series bound on $\|\Sigma_{\phi^\prime}^{-1}\|_{\mathrm{op}}$ (see e.g.~\cite{horn2012matrix}, Corollary~5.6.16), yields the statement.
\end{proof}
\begin{proposition}
\label{prop:F_lip}
Let $F_\phi(\theta_{\mathrm{fid}}) = J_\phi^\top \Sigma_\phi^{-1} J_\phi$ be the Fisher information matrix as in Equation~\eqref{eq:fisher_gauss}. Under Assumptions~\ref{ass:standing}(i)--(iv), it holds that $\phi \mapsto F_\phi(\theta_{\mathrm{fid}})$ is locally Lipschitz on $\Phi$. More precisely, for all $\phi, \phi^\prime \in \Phi$ with $\varepsilon \coloneqq \|\phi^\prime - \phi\|$ satisfying $\kappa_\Sigma \delta_\Sigma < 1$, where $\delta_\Sigma = 4 L_s M \varepsilon$, it holds that
\[
\|F_{\phi'}(\theta_{\mathrm{fid}}) - F_\phi(\theta_{\mathrm{fid}})\|_{\mathrm{op}} \leq \kappa_\Sigma L_J \varepsilon \left(2\|J_\phi\|_{\mathrm{op}} + L_J \varepsilon\right) + \frac{\kappa_\Sigma^2 \delta_\Sigma}{1 - \kappa_\Sigma \delta_\Sigma} \left(\|J_\phi\|_{\mathrm{op}} + L_J \varepsilon\right)^2.
\]
\end{proposition}
\begin{proof}
For ease of notation, let $\varepsilon \coloneqq \|\phi^\prime - \phi\|$, $P \coloneqq \Sigma_\phi^{-1}$, $P^\prime \coloneqq \Sigma_{\phi^\prime}^{-1}$, $J \coloneqq J_\phi$, $J^\prime \coloneqq J_{\phi^\prime}$, $\Delta J \coloneqq J^\prime - J$, and $\Delta P \coloneqq P^\prime - P$. We have that
\begin{align*}
    F_{\phi'}(\theta_{\mathrm{fid}}) - F_\phi(\theta_{\mathrm{fid}}) & = {J^\prime}^\top P^\prime J^\prime - J^\top P J \\
    & = (J + \Delta J)^\top (P + \Delta P)(J + \Delta J) - J^\top P J \\
    & = (J + \Delta J)^\top ( P J + P \Delta J + \Delta P J + \Delta P \Delta J) - J^\top P J.
\end{align*}
Hence, we have that
\begin{align*}
F_{\phi'}(\theta_{\mathrm{fid}}) - F_\phi(\theta_{\mathrm{fid}}) = & J^\top P \, \Delta J + J^\top \Delta P \, J + J^\top \Delta P \, \Delta J + (\Delta J)^\top P \, J \\
& + (\Delta J)^\top P \, \Delta J + (\Delta J)^\top \Delta P \, J + (\Delta J)^\top \Delta P \, \Delta J.
\end{align*}
We recall that $\|P\|_{\mathrm{op}} \leq \kappa_\Sigma$ by hypothesis~(iii), $\|\Delta J\|_{\mathrm{op}} \leq L_J \varepsilon$ by hypothesis~(iv), and $\|\Delta P\|_{\mathrm{op}} \leq \kappa_\Sigma^2 \delta_\Sigma / (1 - \kappa_\Sigma \delta_\Sigma)$ by Proposition~\ref{prop:Cinv}. Hence, we get that
\[
\|F_{\phi'}(\theta_{\mathrm{fid}}) - F_\phi(\theta_{\mathrm{fid}})\|_{\mathrm{op}} \leq \kappa_\Sigma L_J \varepsilon \left(2\|J\|_{\mathrm{op}} + L_J \varepsilon\right) + \frac{\kappa_\Sigma^2 \delta_\Sigma}{1 - \kappa_\Sigma \delta_\Sigma} \left(\|J\|_{\mathrm{op}} + L_J \varepsilon\right)^2
\]
whenever $\kappa_\Sigma \delta_\Sigma < 1$. In particular, $\phi \mapsto F_\phi(\theta_{\mathrm{fid}})$ is locally Lipschitz.
\end{proof}

We are now able to prove the main result.

\begin{proof}[Proof of Theorem~\ref{thm:stability}]
Let $\Delta F \coloneqq F_{\phi'}(\theta_{\mathrm{fid}}) - F_\phi(\theta_{\mathrm{fid}})$ and define $g \colon [0,1] \to \R$ by $g(t) \coloneqq \log \det(F_\phi(\theta_{\mathrm{fid}}) + t \Delta F)$. By Proposition~\ref{prop:F_lip}, whenever $4 \kappa_\Sigma L_s M \varepsilon < 1$ there exists a constant $C_F > 0$ (depending on $\kappa_\Sigma$, $L_s$, $M$, $L_J$, and $\|J_\phi\|_{\mathrm{op}}$) such that $\|\Delta F\|_{\mathrm{op}} \leq C_F \varepsilon$. Set
\[
\varepsilon_0(\phi) \coloneqq \min\!\left\{\frac{1}{4 \kappa_\Sigma L_s M},\; \frac{\lambda_{\min}(F_\phi(\theta_{\mathrm{fid}}))}{C_F}\right\},
\]
so that for $\varepsilon < \varepsilon_0(\phi)$ both Proposition~\ref{prop:F_lip} applies and $\|\Delta F\|_{\mathrm{op}} < \lambda_{\min}(F_\phi(\theta_{\mathrm{fid}}))$. By Weyl's inequality for symmetric matrices, $\lambda_{\min}(F_\phi(\theta_{\mathrm{fid}}) + t \Delta F) \geq \lambda_{\min}(F_\phi(\theta_{\mathrm{fid}})) - t\|\Delta F\|_{\mathrm{op}} > 0$ for all $t \in [0,1]$, so $F_\phi(\theta_{\mathrm{fid}}) + t\Delta F$ is positive definite along the segment and $g$ is differentiable with
\[
g^\prime(t) = \mathrm{tr}\!\left((F_\phi(\theta_{\mathrm{fid}}) + t \Delta F)^{-1} \Delta F\right).
\]
By the mean value theorem, $|\mathcal{L}(\phi^\prime) - \mathcal{L}(\phi)| = |g(1) - g(0)| \leq \sup_{t \in [0,1]} |g^\prime(t)|$. Since $|\mathrm{tr}(C)| \leq p \, \|C\|_{\mathrm{op}}$ for any $C \in \R^{p \times p}$ and
\[
\|(F_\phi(\theta_{\mathrm{fid}}) + t \Delta F)^{-1} \Delta F\|_{\mathrm{op}} \leq \frac{\|\Delta F\|_{\mathrm{op}}}{\lambda_{\min}(F_\phi(\theta_{\mathrm{fid}})) - \|\Delta F\|_{\mathrm{op}}},
\]
we obtain
\begin{equation}
\label{eq:loss_tmh}
|\mathcal{L}(\phi^\prime) - \mathcal{L}(\phi)| \leq \frac{p \, \|\Delta F\|_{\mathrm{op}}}{\lambda_{\min} (F_\phi(\theta_{\mathrm{fid}})) - \|\Delta F\|_{\mathrm{op}}}.
\end{equation}
Combining~\eqref{eq:loss_tmh} with the bound $\|\Delta F\|_{\mathrm{op}} \leq C_F \varepsilon$ from Proposition~\ref{prop:F_lip} yields~\eqref{eq:main}.
\end{proof}

\section{Stability assumptions for TDA pipelines}
\label{app:pipeline_verification}

In this appendix we verify that the pipelines used in
Section~\ref{sec:experiments} satisfy
Assumptions~\ref{ass:standing}(i)--(v) of Theorem~\ref{thm:stability}.
Recall that the summary map takes the form
\[
    s_\phi \;=\; C_\phi \,\circ\, V_\phi \,\circ\, \mathrm{Dgm}_k \,\circ\, f_\phi,
\]
as in Eq.~\eqref{eq:pipeline}. For non-topological pipelines (\IMNN{},
MLP-only, GNN-only, $\log(C_\ell)$+MOPED, peak counts, wavelet scattering), the
persistence and diagram-vectorization stages collapse to the identity. For
pipelines with no trainable parameters (Cubical-PI, Cubical-Silhouette,
Cubical-Curves, $\log(C_\ell)$+MOPED, peak counts, wavelet scattering, DTM),
$s_\phi$ does not depend on $\phi$, so Assumptions~(i) and~(iv) hold trivially
with $L_s = 0$ and $L_J = 0$. We focus the verification on the components
that introduce $\phi$-dependence, and discuss~(ii), (iii) and~(v) for all
pipelines, learnable or not. We present the verification for a single
homological dimension $k$; concatenating multiple dimensions, or multiple
tomographic bins as in Section~\ref{sec:lensing}, preserves each assumption
since each is closed under direct sum of summaries.

\paragraph{Assumption~\ref{ass:standing}(i): Lipschitz continuity of $s_\phi$ in $\phi$.}
The map $\phi \mapsto s_\phi(X)$ is a composition of stages, each Lipschitz
in its input. We discuss the contribution of each stage in turn.

\emph{Filtrations.} The fixed filtrations used in this work (the cubical
filtration, the angular power spectrum used in the lensing baseline, and the
DTM vertex function on the spiral) do not depend on~$\phi$ and contribute a
constant of zero to the Lipschitz constant. The learnable filtrations are neural networks: the per-point MLP $g_\phi$ and
the GNN of Section~\ref{sec:swiss_roll}, the small CNN prepended to the cubical
filtration in TF-CNN-PersLay (Appendix~\ref{app:arch_filtrations}), and the
strided CNN encoders used in the \IMNN{} and in the non-TDA baselines MLP-only
and GNN-only. Each of these is a finite composition of affine maps with Lipschitz
activations (LeakyReLU, GELU); on a compact parameter set $\Phi$ and a
bounded data domain~$K$, the gradient $\nabla_\phi f_\phi(x)$ is uniformly
bounded over $\Phi \times K$, so $\phi \mapsto f_\phi(x)$ is Lipschitz with a
constant independent of $x$. The lower-star extension
$f_\phi(e_{ij}) = \|x_i - x_j\| + \max\{f_\phi(x_i),f_\phi(x_j)\}$ used on the
Alpha complex (Section~\ref{sec:swiss_roll}), and analogously the extension
to higher-dimensional simplices, is the sum of a $\phi$-independent term and
a max of $\phi$-Lipschitz scalars, hence Lipschitz with the same constant as
on vertices.

\emph{Persistent homology.} The map $f \mapsto \mathrm{Dgm}_k(f)$ is
$1$-Lipschitz with respect to the bottleneck distance on diagrams and the
sup-norm on filtration values, by the bottleneck stability theorem of
persistent homology~\cite{cohen2007stability}. For non-topological pipelines
this stage is the identity.

\emph{Vectorizations.} The fixed vectorizations are Lipschitz with respect
to standard distances on diagrams: persistence images are $1$-Lipschitz in
the $1$-Wasserstein distance~\cite{adams2017persistence}; persistence
silhouettes with uniform weighting are Lipschitz in the $\infty$-Wasserstein
distance on bounded support, as means of $1$-Lipschitz tent
functions~\cite{bubenik2015statistical}; the differentiable persistence
curves of Section~\ref{app:arch_vectorizations} are sums of soft-sigmoid
cumulative distribution functions, hence Lipschitz with constant
$O(\tau^{-1})$ on bounded support (where $\tau$ is the soft-sigmoid
temperature reported in Appendix~\ref{app:arch_vectorizations}). The
learnable vectorization PersLay applies a permutation-invariant set network
with spectral normalization on every linear layer
(Appendix~\ref{app:arch_vectorizations}), which constrains each layer's
operator norm to at most one and yields a Lipschitz set network in $W_1$ on
diagrams, with a constant uniformly bounded in $\phi$. The identity vectorization used by the \IMNN{}, $\log(C_\ell)$, peak
counts, and wavelet scattering is trivially $1$-Lipschitz. The non-topological baselines MLP-only and GNN-only end with
sum-pooling over $N$ points, which is Lipschitz with constant $\sqrt{N}$ by
Cauchy--Schwarz.

\emph{Compressors.} The MOPED compressor is a fixed linear map at each gradient step. Its operator norm is bounded by the empirical operator norm of the $C_\phi$ estimate (defined in Appendix \ref{app:gaussianity}), which is in turn controlled by the Hartlap correction and the condition-number monitoring described in Section~\ref{sec:loss}.
The periodic refit of MOPED every $50$
steps (Appendix~\ref{app:arch_compressors}) is treated as an external
procedure and does not enter the within-step Lipschitz analysis. The MLP
heads used as compressors on the spiral and in the pretrained
TF-CNN-PersLay-MLP variant are Lipschitz on compact $\Phi$ by the same neural
network argument used for filtrations.

By composition, $s_\phi$ is Lipschitz in $\phi$ with constant $L_s$ equal to
the product of the stage-wise constants.

\paragraph{Assumption~\ref{ass:standing}(ii): uniform second moment bound.}
All simulators used in this work produce data with bounded support or with
bounded second moment: the spiral is supported on a square
(Section~\ref{sec:swiss_roll}); the GRF is a centered Gaussian field on a
periodic grid with finite total variance (Section~\ref{sec:grf}); the
$\kappa$ maps from \texttt{sbi\_lens} have finite second moment after
Gaussian smoothing (Section~\ref{sec:lensing}). Together with compactness
of $\Phi$ and the Lipschitz bound on $s_\phi$ from~(i), this yields
\[
    M^2 \;\coloneqq\; \sup_{\phi \in \Phi}\; \E_{X \sim p(\cdot|\theta_{\mathrm{fid}})}\!\left[\|s_\phi(X)\|^2\right] \;<\; \infty.
\]
At the level of individual stages: bounded data yield bounded filtration
values, hence persistence diagrams supported in a bounded region of
$\R^2$~\cite{cohen2007stability}; the fixed vectorizations of
Section~\ref{app:arch_vectorizations} (PI, silhouette, persistence curves)
all map bounded diagrams to bounded vectors, and PersLay does so as well
since spectral normalization keeps its outputs bounded under bounded
inputs. The compressor preserves boundedness as the composition of either a
bounded linear map (MOPED) or a Lipschitz MLP on compact $\Phi$.

\paragraph{Assumption~\ref{ass:standing}(iii): bound on $\|\Sigma_\phi^{-1}\|_{\mathrm{op}}$.}
This assumption is verified empirically rather than by a structural argument
on the pipeline. We rely on three diagnostics, run for every reported
configuration: (a)~the sample covariance is computed on a structured batch
sized so that $n_s \gg M$ (Sections~\ref{sec:swiss_roll}--\ref{sec:lensing})
and corrected by the Hartlap factor~\cite{hartlap2007your} for unbiasedness
of the precision matrix; (b)~zero-variance features of the vectorization
output are removed before inversion; (c)~Gaussianity of the final summary is
checked by per-component Kolmogorov--Smirnov tests at $\alpha = 0.05$
(Appendix~\ref{app:gaussianity}). In all reported runs, $\Sigma_\phi$ is
invertible with bounded $\|\Sigma_\phi^{-1}\|_{\mathrm{op}}$.

\paragraph{Assumption~\ref{ass:standing}(iv): Lipschitz continuity of the Jacobian in $\phi$.}
The Jacobian $J_\theta\mu_\phi(\theta_{\mathrm{fid}})$ is estimated by centered
finite differences with step sizes $\Delta = (\Delta_1,\dots,\Delta_d)$
specified in Appendix~\ref{app:arch_training}. Under
Assumption~\ref{ass:standing}(i), Jensen's inequality gives
$\|\mu_{\phi'}(\theta) - \mu_\phi(\theta)\| \leq L_s \|\phi'-\phi\|$ for every
$\theta$ in the stencil. Applying the triangle inequality column-wise to the
finite-difference Jacobian and using $\|A\|_{\mathrm{op}} \leq \|A\|_F$ then
yields
\[
    L_J \;\leq\; \frac{2\sqrt{d}\,L_s}{\Delta_{\min}},
    \qquad \Delta_{\min} \coloneqq \min_{i=1,\dots,d} \Delta_i.
\]

\paragraph{Assumption~\ref{ass:standing}(v): non-degeneracy of $F_\phi(\theta_{\mathrm{fid}})$.}
Since $\Sigma_\phi^{-1}$ is positive definite by~(iii), positive definiteness
of $F_\phi(\theta_{\mathrm{fid}}) = J_\theta\mu_\phi(\theta_{\mathrm{fid}})^\top
\Sigma_\phi^{-1} J_\theta\mu_\phi(\theta_{\mathrm{fid}})$ is equivalent to
$J_\theta\mu_\phi(\theta_{\mathrm{fid}})$ having full column rank, i.e.\ the
summary being sensitive to all $d$ physical parameters at the fiducial. This
is verified empirically in all reported configurations: the marginal
uncertainties $\sigma(\theta_i) = (F_\phi^{-1})_{ii}^{1/2}$ reported in
Sections~\ref{sec:swiss_roll}--\ref{sec:lensing} are finite and of the
expected order of magnitude, and the condition number of $F_\phi$ is
monitored throughout training (see~(iii) above and the diagnostics in
Appendix~\ref{app:gaussianity}).

\section{Noisy Spiral Total Fisher}
\label{app:swiss_roll_fisher}
This appendix derives the total Fisher information matrix for the 2D noisy spiral generative model. Points in $\R^2$ are sampled i.i.d. from a known two-component mixture, with a background component $p_B$ and a spiral component $p_S$. Given a mixing weight $\lambda \in (0,1)$, the point cloud $\mathcal{X} = \{x_k\}_{k=1}^N \subseteq \R^2$ is obtained by independently sampling each point from
\begin{equation}
\label{eq:swiss_mixture}
    p(x; \theta) = \lambda p_B(x) + (1-\lambda) p_S(x; \theta),
\end{equation}
where $\theta = (\mu, w)^\top$ are the target parameters: $\mu$ controls the radial expansion of the spiral and $w$ its width. The background density $p_B$ does not depend on $\theta$ and is uniform on a square $\mathcal{S} = [-h, h]^2$ of area $4h^2$. Since the spiral component $p_S$ is most naturally expressed in polar coordinates $(\rho, \Phi)$, with $x = \rho \cos \Phi$ and $y = \rho \sin \Phi$, we work in this system throughout. The Jacobian determinant of the change of variables is $|J| = \rho$, and in these coordinates the square $\mathcal{S}$ is described by
\begin{equation}
\label{eq:swiss_polar_support}
    \mathcal{S}^\prime = \left\{ (\rho, \Phi) : 0 \leq \Phi < 2\pi, \; 
    0 \leq \rho \leq \frac{h}{\max(|\cos \Phi|, |\sin \Phi|)} \right\}.
\end{equation}
The background density therefore reads
\begin{equation}
\label{eq:swiss_pB_polar}
    p_B(\rho, \Phi) = \frac{\rho}{4 h^2}\, 
    \mathbb{I}_{\mathcal{S}^\prime}(\rho, \Phi),
\end{equation}
where $\mathbb{I}_A$ denotes the indicator function of the set $A$. Let us now focus on the spiral component $p_S$. The spiral is sampled hierarchically as follows:
\begin{enumerate}
    \item an angular coordinate $t$ is drawn uniformly on $[0, 4\pi)$, so that the spiral completes two full turns;
    \item given $t$, the radial coordinate is perturbed by Gaussian noise, $\rho \mid t \sim \mathcal{N}(\mu t, w^2)$;
    \item the sampled point is mapped to Cartesian coordinates via $(x, y) = (\rho \cos t, \rho \sin t)$.
\end{enumerate}
Since $t$ and $t + 2\pi$ map to the same point in $\R^2$, to obtain an injective map we reparametrize the generative process on $\Phi \in [0, 2\pi)$ by marginalizing over the winding. Setting $\Phi \coloneqq t \bmod 2\pi$, the two preimages $t \in \{\Phi, \Phi + 2\pi\}$ are equally likely conditionally on $\Phi$, and the hierarchical sampling becomes $\Phi \sim \mathcal{U}(0, 2\pi)$, and $\rho \mid \Phi \sim \frac{1}{2}\,\mathcal{N}(\mu \Phi, w^2) + \frac{1}{2}\,\mathcal{N}(\mu(\Phi + 2\pi), w^2)$ \footnote{The Gaussian conditional $\rho | \Phi$ places a small amount of mass on the region $\rho < 0$. For the fiducial $\mu = 0.6$, $w = 0.1$, this mass, concentrated near $\Phi\rightarrow0^+$, is bounded by $10^{-2}$ and we neglect it in the following.}. The spiral density therefore reads
\begin{equation}
\label{eq:swiss_pS_polar}
    p_S(\rho, \Phi \mid \theta) = \frac{1}{2\pi} \cdot \frac{1}{2\sqrt{2\pi}\, w} \left[ \exp\left(-\frac{(\rho - \mu \Phi)^2}{2w^2}\right) + \exp\left(-\frac{(\rho - \mu (\Phi + 2\pi))^2}{2w^2}\right) \right] \mathbb{I}_{\mathcal{S}^{\prime\prime}}(\rho, \Phi),
\end{equation}
supported on $\mathcal{S}^{\prime\prime} = \{(\rho, \Phi) \colon 0 \leq \Phi < 2\pi, \; \rho \geq 0\}$. Since the sample is i.i.d.\ from $p(\cdot; \theta)$, the total Fisher information matrix in polar coordinates reads~\cite{kay1993statistical}
\begin{equation}
\label{eq:swiss_fisher_def}
    F_X[{i,j}](\theta) = N \iint 
    \frac{\partial_{\theta_i} p(\rho, \Phi; \theta)\, 
    \partial_{\theta_j} p(\rho, \Phi; \theta)}{p(\rho, \Phi; \theta)}\, 
    d\rho\, d\Phi.
\end{equation}
Since only $p_S$ depends on $\theta$, we have $\partial_\theta p = (1-\lambda)\,\partial_\theta p_S$, and \eqref{eq:swiss_fisher_def} specializes to
\begin{equation}
\label{eq:swiss_fisher_spiral}
    F_X[{i,j}](\theta) = N \iint 
    \frac{(1-\lambda)^2\,\partial_{\theta_i} p_S\,\partial_{\theta_j} p_S}
    {\lambda\, p_B + (1-\lambda)\, p_S}\, d\rho\, d\Phi,
\end{equation}
where the dependence on $(\rho, \Phi, \theta)$ has been suppressed for readability. To lighten the notation, we introduce $\zeta_1 = \frac{\rho - \mu\Phi}{w}$, $\zeta_2 = \frac{\rho - \mu(\Phi+2\pi)}{w}$, and $g_\ell = \frac{1}{\sqrt{2\pi}\,w}e^{-\zeta_\ell^2/2}$ for $\ell=1, 2$. A direct computation then yields
\begin{align}
    \partial_w\, p_S &= \frac{1}{2\pi}\cdot\frac{1}{2} 
    \left[ g_1\, \frac{\zeta_1^2 - 1}{w} + g_2\, \frac{\zeta_2^2 - 1}{w} \right],
    \label{eq:swiss_dw_pS} \\[4pt]
    \partial_\mu\, p_S &= \frac{1}{2\pi}\cdot\frac{1}{2} 
    \left[ g_1\, \frac{\zeta_1\, \Phi}{w} + g_2\, \frac{\zeta_2\,(\Phi + 2\pi)}{w} \right].
    \label{eq:swiss_dmu_pS}
\end{align}
The total Fisher information matrix $F_X(\theta)$ is then obtained by substituting \eqref{eq:swiss_dw_pS}--\eqref{eq:swiss_dmu_pS} into \eqref{eq:swiss_fisher_spiral}. The resulting integral admits no known closed form and is evaluated numerically.

\section{Additional Noisy Spiral Results}
\label{app:swiss_roll_results}

We report the results for the spiral experiment for the other four fiducial parameter values in Table~\ref{tab:swiss_roll_grid}. In Figure \ref{fig:learned_filtration} we plot the vertex values and the corresponding filtration learned by TF-TDA-MLP method.

\begin{table}[t]
\centering
\scriptsize
\setlength{\tabcolsep}{3.5pt}
\renewcommand{\arraystretch}{1.08}
\caption{Complete spiral results for all fiducial configurations
$(\mu_{\text{fid}},w_{\text{fid}})$, averaged over $5$ seeds. Entries are
mean $\pm$ standard deviation of $\log|F|$.}
\label{tab:swiss_roll_grid}
\begin{tabular*}{\textwidth}{@{\extracolsep{\fill}}lccccc@{}}
\toprule
\textbf{Method} & $(0.6,0.1)$ & $(0.8,0.1)$ & $(0.7,0.2)$ & $(0.5,0.2)$ & $(0.7,0.3)$ \\
\midrule
\textit{Total}
& \textit{24.27} & \textit{24.24} & \textit{21.36} & \textit{21.41} & \textit{19.64} \\
\midrule
\tabgroup{6}{Fixed topological ablation: persistence without TopoFisher learning}
\midrule
DTM fixed filtration
& $17.18 \pm 0.06$ & $16.34 \pm 0.25$ & $15.83 \pm 0.09$ & $16.65 \pm 0.06$ & $15.16 \pm 0.05$ \\
\midrule
\tabgroup{6}{TopoFisher summaries: our learned topological contributions}
\midrule
\textbf{TF-TDA-MLP}
& $\mathbf{18.99 \pm 0.30}$ & $17.47 \pm 0.63$ & $\mathbf{17.50 \pm 0.28}$ & $18.32 \pm 0.37$ & $\mathbf{16.66 \pm 0.07}$ \\
TF-TDA-GNN
& $16.20 \pm 0.19$ & $15.57 \pm 0.37$ & $15.14 \pm 0.14$ & $16.63 \pm 0.43$ & $14.72 \pm 0.11$ \\
\midrule
\tabgroup{6}{Unconstrained neural references: non-topological summaries}
\midrule
MLP-only
& $18.96 \pm 0.58$ & $\mathbf{17.88 \pm 0.49}$ & $17.48 \pm 0.11$ & $\mathbf{18.40 \pm 0.21}$ & $16.54 \pm 0.19$ \\
GNN-only
& $18.26 \pm 0.25$ & $17.45 \pm 0.12$ & $16.78 \pm 0.16$ & $17.68 \pm 0.06$ & $15.87 \pm 0.08$ \\
\bottomrule
\end{tabular*}
\end{table}

To highlight the interpretability of the TF-TDA-MLP, we compare its learned filtration with the vertex mapping produced by the MLP-only baseline. Since the latter maps each point $x$ to a two-dimensional statistic $f_\phi(x) \in \mathbb{R}^2$ before sum-pooling, we visualize both components in Fig. \ref{fig:swiss_roll_NN}. 

\begin{figure}[t]
     \centering
     \begin{subfigure}[b]{0.33\textwidth}
         \centering
         \includegraphics[width=\textwidth]{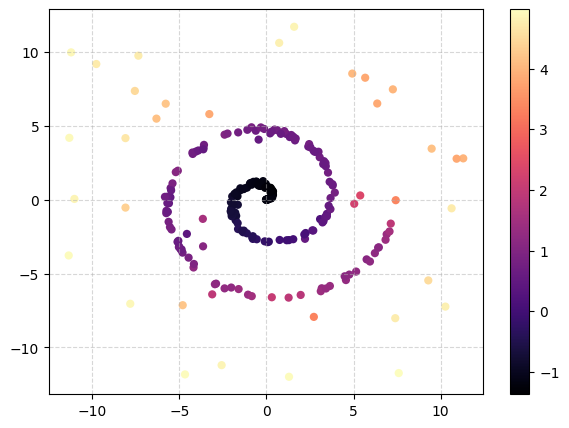}
         \caption{Learned vertex values $f_\phi(x)$.}
     \end{subfigure}
     \hfill
     \begin{subfigure}[b]{0.65\textwidth}
         \centering
         \includegraphics[width=\textwidth]{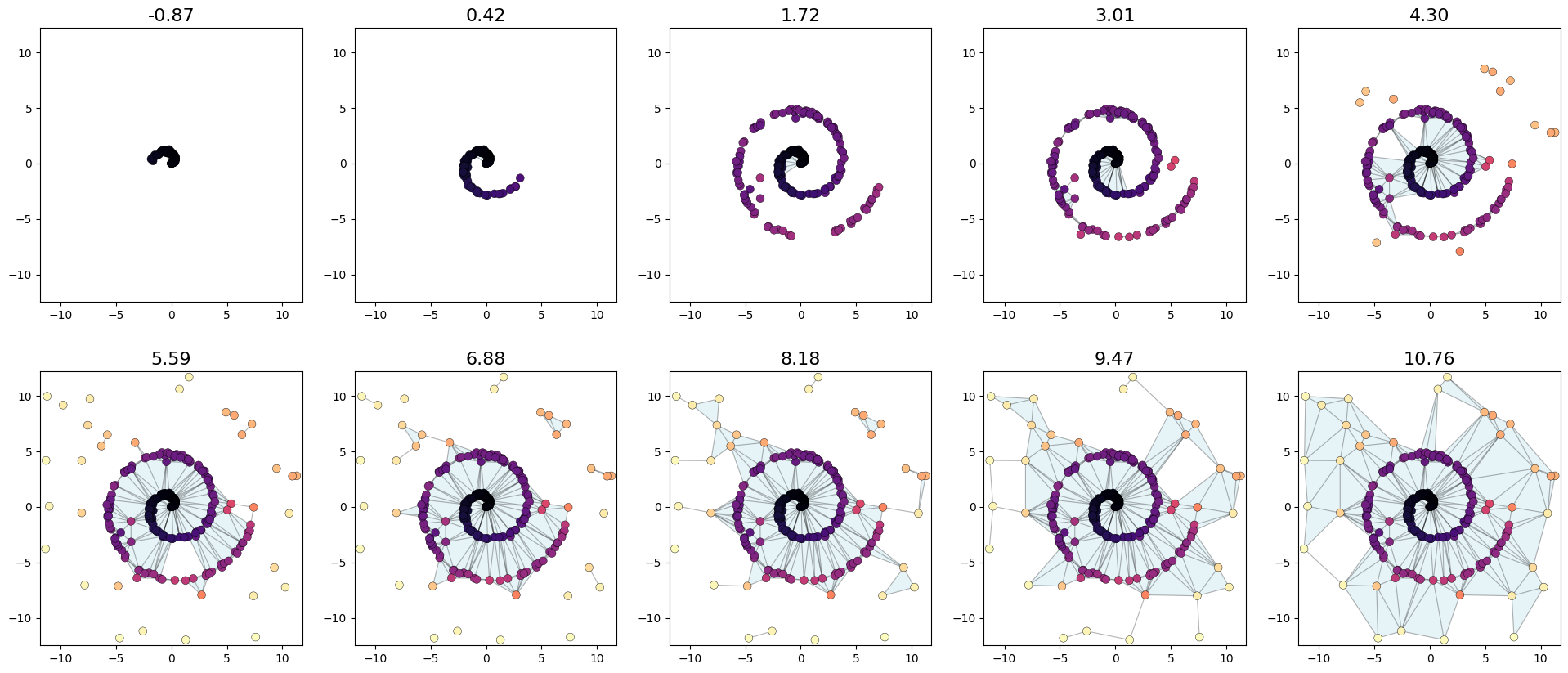}
         \caption{Resulting filtration.}
     \end{subfigure}
     \caption{Visualization of the TF-TDA-MLP pipeline. (a) The model assigns filtration values to vertices to "unroll" the spiral geometry. (b) The corresponding filtration effectively prioritizes the manifold while filtering background noise.}
     \label{fig:learned_filtration}
\end{figure}

In contrast to the topological approach, the purely neural baseline lacks a clear geometric interpretation. The network appears to concentrate the information on a small subset of points that assume extreme values to drive the sum-pooling operation, but their spatial distribution does not correspond to identifiable features of the manifold.

\begin{figure}[t]
     \centering
     \begin{subfigure}[b]{0.4\textwidth}
         \centering
         \includegraphics[width=\textwidth]{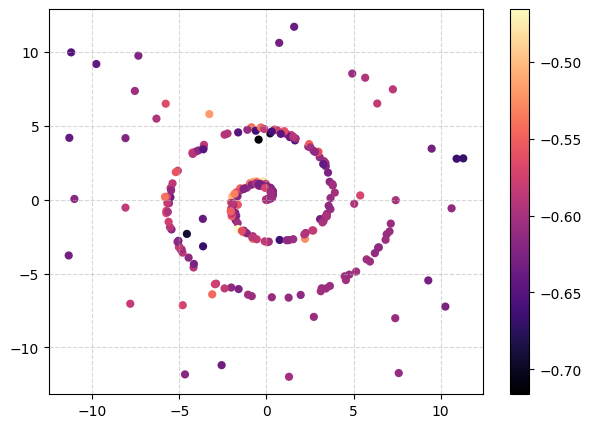}
         \caption{Component $f_\phi(x)^1$.}
     \end{subfigure}
     \hspace{0.3cm} % Regola questo valore per avvicinarle o allontanarle
     \begin{subfigure}[b]{0.4\textwidth}
         \centering
         \includegraphics[width=\textwidth]{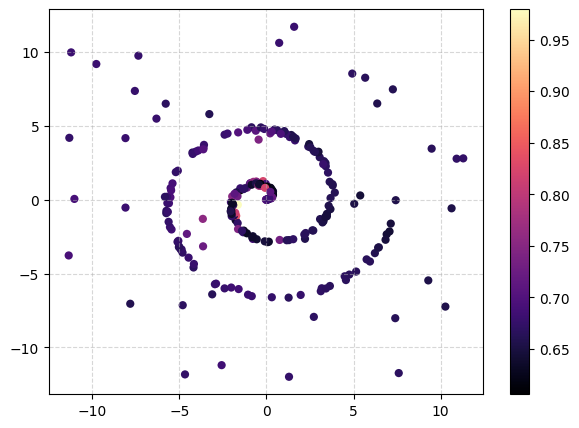}
         \caption{Component $f_\phi(x)^2$.}
     \end{subfigure}
     \caption{Vertex-wise outputs for the two components learned by the MLP-only baseline.}
     \label{fig:swiss_roll_NN}
\end{figure}

\section{Total Fisher information for the GRF benchmark}\label{app:totalGRF}
Let $x\in\mathbb{R}^{N\times N}$ be a zero-mean isotropic and homogeneous Gaussian random field on a periodic grid, and let $\tilde x_{\mathbf{k}}$ denote its discrete Fourier coefficients. Under the simulator used in Section~\ref{sec:grf}, the field is fully specified by the isotropic power spectrum
\begin{equation}
    P(k;\theta)=A_s\left(\frac{k}{\kpiv}\right)^{-B},
    \qquad \theta=(A_s,B),
\end{equation}
with $k=\|\mathbf{k}\|$ and $\kpiv$ the geometric mean of the nonzero Fourier radii. Since the field is Gaussian with zero mean, the exact data likelihood is multivariate Gaussian with diagonal covariance $\Sigma(\theta)$ in Fourier space. The Fisher matrix is therefore
\begin{equation}
    F_{ij}(\theta)
    = \frac{1}{2}\,\mathrm{Tr}\!\left[
        \Sigma^{-1}(\theta)\,\partial_{\theta_i}\Sigma(\theta)\,
        \Sigma^{-1}(\theta)\,\partial_{\theta_j}\Sigma(\theta)
    \right].
    \label{eq:grf_fisher_trace}
\end{equation}

Because $\Sigma(\theta)$ is diagonal in the Fourier basis, Eq.~\eqref{eq:grf_fisher_trace} reduces to a sum over modes:
\begin{equation}
    F_{ij}(\theta)
    = \frac{1}{2}\sum_{\mathbf{k}\neq 0}
    \frac{\partial_{\theta_i}P(k;\theta)\,\partial_{\theta_j}P(k;\theta)}
         {P(k;\theta)^2}
    = \frac{1}{2}\sum_{\mathbf{k}\neq 0}
    \partial_{\theta_i}\log P(k;\theta)\,
    \partial_{\theta_j}\log P(k;\theta).
    \label{eq:grf_fisher_modes}
\end{equation}
Here the sum runs over all nonzero discrete Fourier modes of the $N\times N$ lattice; the prefactor $1/2$ accounts for the reality condition $\tilde x_{-\mathbf{k}}=\tilde x_{\mathbf{k}}^\ast$. Equivalently, one may sum only over an independent half-plane of modes and drop the factor $1/2$.

For the power-law spectrum used here,
\begin{equation}
    \partial_{A_s}\log P(k;\theta)=\frac{1}{A_s},
    \qquad
    \partial_B \log P(k;\theta)= -\log\!\left(\frac{k}{\kpiv}\right).
\end{equation}
Substituting into Eq.~\eqref{eq:grf_fisher_modes} gives
\begin{align}
    F_{A_sA_s}
    &= \frac{1}{2A_s^2}\sum_{\mathbf{k}\neq 0} 1
    = \frac{N_{\rm nz}}{2A_s^2},
    \label{eq:grf_faa}\\
    F_{A_sB}
    &= -\frac{1}{2A_s}\sum_{\mathbf{k}\neq 0}
    \log\!\left(\frac{k}{\kpiv}\right),
    \label{eq:grf_fab}\\
    F_{BB}
    &= \frac{1}{2}\sum_{\mathbf{k}\neq 0}
    \log^2\!\left(\frac{k}{\kpiv}\right),
    \label{eq:grf_fbb}
\end{align}
where $N_{\rm nz}=N^2-1$ is the number of nonzero Fourier modes.

The choice of pivot scale
\begin{equation}
    \kpiv
    = \exp\!\left(
        \frac{1}{N_{\rm nz}}
        \sum_{\mathbf{k}\neq 0}\log k
      \right)
\end{equation}
makes the cross term vanish exactly:
\begin{equation}
    \sum_{\mathbf{k}\neq 0}\log\!\left(\frac{k}{\kpiv}\right)=0
    \qquad\Longrightarrow\qquad
    F_{A_sB}=0.
\end{equation}
Thus the total Fisher matrix is diagonal,
\begin{equation}
    F(\theta)=
    \begin{pmatrix}
        \dfrac{N_{\rm nz}}{2A_s^2} & 0\\[1.2ex]
        0 & \dfrac{1}{2}\sum_{\mathbf{k}\neq 0}\log^2\!\left(\dfrac{k}{\kpiv}\right)
    \end{pmatrix}.
    \label{eq:grf_fisher_final}
\end{equation}
Two useful consequences follow immediately. First, the bound is independent of the fiducial slope $B_0$, since $\partial_B\log P$ does not depend on $B$. Second, the overall box-size normalization cancels in $k/\kpiv$, so only the discrete lattice of mode radii matters.

For our benchmark, $N=64$, hence $N_{\rm nz}=64^2-1=4095$ and
\begin{equation}
    F_{A_sA_s}=\frac{4095}{2}=2047.5
    \qquad\Longrightarrow\qquad
    \sigma(A_s)=F_{A_sA_s}^{-1/2}=0.0221.
\end{equation}
The second diagonal entry is obtained by evaluating the lattice sum in Eq.~\eqref{eq:grf_fbb} numerically on the nonzero $64\times 64$ Fourier grid, yielding
\begin{equation}
    F_{BB}\approx 5.16\times 10^2,
    \qquad
    \sigma(B)=F_{BB}^{-1/2}\approx 0.044.
\end{equation}
Finally,
\begin{equation}
    \logdetF_{\rm tot}
    = \log\!\bigl(F_{A_sA_s}F_{BB}\bigr)
    \approx 13.89,
\end{equation}
which is the denominator of efficiency used throughout Section~\ref{sec:grf}.

\section{GRF: per-\texorpdfstring{$B$}{B} results}
\label{app:grf_all_B}

This appendix reports the per-spectral-index results for the GRF benchmark of
Section~\ref{sec:grf}. All entries are computed at $N=64$ and averaged over
$5$ independent seeds for each fiducial value
$B\in\{-2,-1,0,1,2\}$. The theoretical optimum is
$\log|F_X|=13.89$ for every value of $B$.

\begin{table}[ht]
\centering
\scriptsize
\setlength{\tabcolsep}{3.5pt}
\renewcommand{\arraystretch}{1.08}
\caption{GRF results for all spectral indices $B \in \{-2,-1,0,1,2\}$,
$N=64$, averaged over $5$ seeds. Entries are mean $\pm$ standard deviation of
$\log|F|$. The theoretical optimum is $\log|F_X|=13.89$ for every $B$.}
\label{tab:grf_all_B}
\begin{tabular*}{\textwidth}{@{\extracolsep{\fill}}lccccc@{}}
\toprule
\textbf{Method} & $B=-2$ & $B=-1$ & $B=0$ & $B=1$ & $B=2$ \\
\midrule
\textit{Total}
& \multicolumn{5}{c}{\textit{13.89}} \\
\midrule
\tabgroup{6}{Fixed topological ablations: persistence without TopoFisher learning}
\midrule
Cubical-PI
& $10.573 \pm 0.009$ & $11.314 \pm 0.007$ & $11.996 \pm 0.006$ & $12.085 \pm 0.006$ & $11.682 \pm 0.010$ \\
Cubical-Silhouette
& $9.895 \pm 0.012$ & $10.597 \pm 0.012$ & $11.304 \pm 0.005$ & $11.256 \pm 0.005$ & $10.369 \pm 0.010$ \\
Cubical-Curves
& $7.073 \pm 0.019$ & $8.685 \pm 0.007$ & $10.372 \pm 0.006$ & $10.486 \pm 0.008$ & $9.467 \pm 0.009$ \\
\midrule
\tabgroup{6}{TopoFisher summaries: our learned topological contributions}
\midrule
\textbf{TF-Cubical-PersLay}
& $\mathbf{13.011 \pm 0.013}$ & $\mathbf{12.880 \pm 0.003}$ & $\mathbf{12.756 \pm 0.007}$ & $\mathbf{12.520 \pm 0.005}$ & $11.980 \pm 0.008$ \\
TF-CNN-PersLay
& $11.920 \pm 0.120$ & $12.078 \pm 0.081$ & $12.189 \pm 0.081$ & $12.205 \pm 0.046$ & $11.936 \pm 0.103$ \\
\midrule
\tabgroup{6}{Unconstrained Fisher-neural reference: non-topological summary}
\midrule
\IMNN{}
& $12.142 \pm 0.040$ & $12.270 \pm 0.033$ & $12.34 \pm 0.043$ & $12.34 \pm 0.047$ & $\mathbf{12.29 \pm 0.092}$ \\
\bottomrule
\end{tabular*}
\end{table}

\begin{table}[ht]
\centering
\scriptsize
\setlength{\tabcolsep}{3.5pt}
\renewcommand{\arraystretch}{1.08}
\caption{GRF marginal uncertainties across spectral indices, averaged over
$5$ seeds. Each cell reports $\sigma(A_s)/\sigma(B)$. The theoretical values
are $\sigma(A_s)=0.022$ and $\sigma(B)=0.044$ for every $B$.}
\label{tab:grf_sigmas}
\begin{tabular*}{\textwidth}{@{\extracolsep{\fill}}lccccc@{}}
\toprule
\textbf{Method} & $B=-2$ & $B=-1$ & $B=0$ & $B=1$ & $B=2$ \\
\midrule
\textit{Total}
& \multicolumn{5}{c}{\textit{$0.022/0.044$}} \\
\midrule
\tabgroup{6}{Fixed topological ablations: persistence without TopoFisher learning}
\midrule
Cubical-PI
& $0.0436/0.1445$ & $0.0349/0.1049$ & $0.0327/0.0759$ & $0.0345/0.0690$ & $0.0405/0.0723$ \\
Cubical-Silhouette
& $0.0529/0.1967$ & $0.0381/0.1445$ & $0.0349/0.1006$ & $0.0411/0.0916$ & $0.0625/0.1104$ \\
Cubical-Curves
& $0.1845/0.7144$ & $0.0644/0.3410$ & $0.0379/0.1480$ & $0.0524/0.1190$ & $0.0929/0.1376$ \\
\midrule
\tabgroup{6}{TopoFisher summaries: our learned topological contributions}
\midrule
\textbf{TF-Cubical-PersLay}
& $0.0284/0.0570$ & $0.0287/0.0596$ & $0.0298/0.0606$ & $0.0325/0.0611$ & $0.0380/0.0672$ \\
TF-CNN-PersLay
& $0.0334/0.0775$ & $0.0339/0.0703$ & $0.0335/0.0674$ & $0.0333/0.0672$ & $0.0356/0.0720$ \\
\midrule
\tabgroup{6}{Unconstrained Fisher-neural reference: non-topological summary}
\midrule
\IMNN{}
& $0.0322/0.0693$ & $0.0320/0.0657$ & $0.0322/0.0636$ & $0.0321/0.0639$ & $0.0320/0.0636$ \\
\bottomrule
\end{tabular*}
\end{table}
\section{Further Lensing results}
\label{app:lensing_versions}

This appendix provides the per-bin and transfer-learning details for the weak
lensing experiment of Section~\ref{sec:lensing}. All results use the
$\sigma=2'$ smoothed $512\times512$ convergence maps and the parameter pair
$(\Om,\sig)$. For each tomographic bin and each method, we report the mean and
standard deviation across $5$ independent seeds. The per-bin quantities in
Tables~\ref{tab:lensing_per_bin_logF} are useful
for diagnosing how information changes with source redshift, but the
Survey-level constraints reported in the main text are obtained by adding the
Fisher matrices across bins, and then computing $\log|F_{\mathrm{tot}}|$ and
$\sqrt{(F_{\mathrm{tot}}^{-1})_{ii}}$.

\paragraph{Per-bin behavior.}
The per-bin log-determinants in Table~\ref{tab:lensing_per_bin_logF} show three
main trends. First, the standard $\log(C_\ell)$ baseline is comparatively flat
across bins, with $\log|F| \simeq 10.2$--$10.6$. Second, fixed topological
summaries recover substantially more information than the power spectrum in the
higher-redshift bins, where the lensing signal is stronger. Third, the learned
diagram vectorization  TF-Cubical-PersLay improves consistently over all fixed
topological vectorizations and increases monotonically from bin~$0$ to bin~$4$,
from $\log|F|=16.71$ to $\log|F|=18.99$. The unconstrained \IMNN{} baseline is
weaker in the first bin but becomes the strongest in the highest-redshift bins,
reaching $\log|F|=20.05$ in bin~$4$.

\begin{table}[ht]
\centering
\scriptsize
\setlength{\tabcolsep}{3.5pt}
\renewcommand{\arraystretch}{1.08}
\caption{Weak lensing per-bin results for $(\Om,\sig)$ at $\sigma=2'$
smoothing. Each cell reports mean $\pm$ standard deviation of $\log|F|$ across
$5$ seeds. Bin index increases with effective source redshift. Transfer rows
are trained on lognormal maps and evaluated on LPT maps without retraining.}
\label{tab:lensing_per_bin_logF}
\resizebox{\textwidth}{!}{%
\begin{tabular}{@{}lccccc@{}}
\toprule
\textbf{Method} & bin~$0$ & bin~$1$ & bin~$2$ & bin~$3$ & bin~$4$ \\
\midrule
\multicolumn{6}{@{}l}{\emph{External cosmology baseline}}\\
\midrule
$\log(C_\ell)$+MOPED
& $10.198 \pm 0.086$ & $10.574 \pm 0.094$ & $10.583 \pm 0.080$ & $10.499 \pm 0.072$ & $10.279 \pm 0.069$ \\
\midrule
\multicolumn{6}{@{}l}{\emph{Fixed topological ablations: persistence without TopoFisher 
learning}}\\
\midrule
Cubical-PI
& $12.672 \pm 0.017$ & $15.165 \pm 0.025$ & $16.055 \pm 0.019$ & $16.364 \pm 0.020$ & $16.346 \pm 0.025$ \\
Cubical-Curves
& $8.935 \pm 0.249$ & $13.448 \pm 0.074$ & $14.912 \pm 0.015$ & $15.464 \pm 0.014$ & $15.296 \pm 0.017$ \\
Cubical-Silhouette
& $13.048 \pm 0.019$ & $13.053 \pm 0.020$ & $12.673 \pm 0.019$ & $12.681 \pm 0.016$ & $13.259 \pm 0.015$ \\
\midrule
\multicolumn{6}{@{}l}{\emph{TopoFisher summaries: our learned topological contributions}}\\
\midrule
\textbf{TF-Cubical-PersLay}
& $\mathbf{16.705 \pm 0.019}$ & $\mathbf{17.853 \pm 0.012}$ & $18.286 \pm 0.019$ & $18.591 \pm 0.032$ & $18.988 \pm 0.029$ \\
TF-CNN-PersLay
& $14.755 \pm 0.133$ & $17.166 \pm 0.166$ & $17.272 \pm 0.608$ & $17.171 \pm 0.177$ & $16.965 \pm 0.514$ \\
\midrule
\multicolumn{6}{@{}l}{\emph{Unconstrained Fisher-neural reference: non-topological summary}}\\
\midrule
\IMNN{}
& $14.893 \pm 0.097$ & $17.624 \pm 0.111$ & $\mathbf{18.877 \pm 0.051}$ & $\mathbf{19.694 \pm 0.084}$ & $\mathbf{20.053 \pm 0.091}$ \\
\midrule
\multicolumn{6}{@{}l}{\emph{Transfer Fisher: lognormal-trained summaries evaluated on LPT maps}}\\
\midrule
TF-Cubical-PersLay
& $10.785 \pm 0.039$ & $13.608 \pm 0.014$ & $\mathbf{16.565 \pm 0.038}$ & $\mathbf{16.657 \pm 0.019}$ & $\mathbf{17.011 \pm 0.014}$ \\
TF-CNN-PersLay
& $1.036 \pm 1.292$ & $5.863 \pm 0.372$ & $7.807 \pm 0.638$ & $9.508 \pm 0.370$ & $12.123 \pm 2.380$ \\
\IMNN{}
& $-8.519 \pm 6.562$ & $-2.030 \pm 2.298$ & $3.388 \pm 2.321$ & $2.601 \pm 2.618$ & $3.710 \pm 1.623$ \\
\bottomrule
\end{tabular}}
\end{table}

\begin{table}[ht]
\centering
\scriptsize
\setlength{\tabcolsep}{3.5pt}
\renewcommand{\arraystretch}{1.08}
\caption{Weak lensing per-bin marginal uncertainties for in-distribution
lognormal simulations. Each cell reports $\sigma(\Om)/\sigma(\sig)$ averaged
over $5$ seeds. These are per-bin constraints only; survey-level constraints
are obtained by adding Fisher matrices across bins.}
\label{tab:lensing_per_bin_sigmas}
\resizebox{\textwidth}{!}{%
\begin{tabular}{@{}lccccc@{}}
\toprule
\textbf{Method} & bin~$0$ & bin~$1$ & bin~$2$ & bin~$3$ & bin~$4$ \\
\midrule
\multicolumn{6}{@{}l}{\emph{External cosmology baseline}}\\
\midrule
$\log(C_\ell)$
& $0.1242/0.1749$ & $0.0975/0.1401$ & $0.0866/0.1333$ & $0.0788/0.1311$ & $0.0712/0.1296$ \\
\midrule
\multicolumn{6}{@{}l}{\emph{Fixed topological ablations: persistence without TopoFisher learning}}\\
\midrule
Cubical-PI
& $0.0848/0.0987$ & $0.0481/0.0461$ & $0.0344/0.0305$ & $0.0299/0.0246$ & $0.0246/0.0178$ \\
Cubical-Curves
& $0.4600/0.4873$ & $0.0904/0.0826$ & $0.0496/0.0421$ & $0.0365/0.0291$ & $0.0306/0.0234$ \\
Cubical-Silhouette
& $0.0837/0.0884$ & $0.0783/0.0801$ & $0.0778/0.0853$ & $0.0676/0.0789$ & $0.0442/0.0529$ \\
\midrule
\multicolumn{6}{@{}l}{\emph{TopoFisher summaries: our learned topological contributions}}\\
\midrule
\textbf{TF-Cubical-PersLay}
& $\mathbf{0.0338/0.0410}$ & $\mathbf{0.0213/0.0227}$ & $0.0192/0.0198$ & $0.0172/0.0183$ & $0.0145/0.0168$ \\
TF-CNN-PersLay
& $0.0417/0.0424$ & $0.0226/0.0183$ & $0.0235/0.0163$ & $0.0232/0.0145$ & $0.0210/0.0132$ \\
\midrule
\multicolumn{6}{@{}l}{\emph{Unconstrained Fisher-neural reference: non-topological summary}}\\
\midrule
\IMNN{}
& $0.0449/0.0452$ & $0.0219/0.0175$ & $\mathbf{0.0157/0.0112}$ & $\mathbf{0.0119/0.0076}$ & $\mathbf{0.0097/0.0051}$ \\
\bottomrule
\end{tabular}}
\end{table}

\paragraph{Transfer-learning diagnostic.}
The transfer experiment is intentionally stringent: all trainable components
are optimized on the lognormal \texttt{sbi\_lens} simulator and then evaluated
without retraining on LPT-based maps at the same fiducial cosmology, resolution,
and smoothing scale. The transfer degradation of the \IMNN{} and
TF-CNN-PersLay is large despite their strong in-distribution Fisher values,
indicating that learned pixel-space neural stages can exploit
simulator-specific features. In contrast, Cubical-PersLay has no learned
pixel-space filtration and transfers much more reliably. Its transferred Fisher
ellipse is still more degenerate than its in-distribution counterpart, as
reflected by the weaker marginal uncertainties, but it preserves a high
determinant and therefore a strong constraint along at least one cosmological
parameter combination.

\paragraph{Details on simulation-based inference studies}

We perform simulation-based inference (also known as likelihood-free inference) using Neural Posterior Estimation \citep[NPE;][]{PapamakariosMurray2016, Lueckmann2017, Greenberg2019}.
Rather than approximating the intractable likelihood $p(d \mid \theta)$, NPE directly targets the posterior by training a conditional density estimator $q_{\varphi}(\theta \mid d)$, parameterised by neural-network weights $\varphi$, on simulated $(\theta, d)$ pairs drawn from the joint $p(\theta, d) = p(\theta)\,p(d \mid \theta)$.
 
We adopt a conditional normalizing flow as the density estimator, specifically a conditional RealNVP \citep{Dinh2017} composed of a stack of affine coupling layers. In each coupling layer, the scale and shift parameters of the affine transformation are produced by a small MLP conditioner that takes as input both the masked half of the parameter vector and the conditioning summary statistic $d$.
 
Training proceeds by minimising the forward Kullback--Leibler divergence between the true posterior and the variational approximation,
\begin{equation}
    \mathcal{L}(\varphi)
    \;=\;
    \mathbb{E}_{p(d)}\!\left[
        D_{\mathrm{KL}}\!\big(
            p(\theta \mid d) \,\big\|\, q_{\varphi}(\theta \mid d)
        \big)
    \right],
\end{equation}
which, up to a $\varphi$-independent constant, is equivalent to maximising the expected conditional log-likelihood of the flow on samples from the joint,
\begin{equation}
    \varphi^{\star}
    \;=\;
    \arg\max_{\varphi}\;
    \mathbb{E}_{p(\theta, d)}\!\left[\log q_{\varphi}(\theta \mid d)\right].
\end{equation}
We estimate this expectation by Monte Carlo over a training set of simulator draws. At inference time we condition the trained flow on the observed summary $d_{\mathrm{obs}}$ and draw samples directly from $q_{\varphi^{\star}}(\theta \mid d_{\mathrm{obs}})$.

We note that we perform our SBI using only auto-bins (i.e. correlations within the same bin). In principle, we could improve performance by including cross-correlations between different redshift bins \citep[compare][]{arxiv:2010.07376}. Likewise, in a realistic cosmological analysis we would compute the cross-correlation between all bin pairs in the power spectrum analysis. This is beyond the scope of demonstrating the feasibility of our TopoFisher pipeline, and uniformly disregarding cross-correlations between bins ensures a fair comparison between summary statistics. We reserve a more realistic lensing optimization that includes cross-correlations between bins for future work. We also note that the SBI posteriors of the Power Spectrum marginally exceed its previously computed Fisher information. This is due to the MOPED compression, which is lossless in theory, but imparts losses when the derivatives and covariance are noisy \citep{Asgari2015}.

We show the inference constraints on our two-parameter scenario in Figure \ref{fig:sbi_lensing_omsig8}, and for the extended 6-parameter inference in Figure \ref{fig:sbi_lensing_allparams}. The $\Om$ in the main text is $\Om = \Omega_\mathrm{c}+\Omega_\mathrm{b}$; the total matter density is the sum of the dark matter and baryonic matter densities.
The prior over cosmological parameters used during SBI training is

\begin{align*}
    \Omega_c &\sim \mathcal{N}_{>0}(0.2664,\, 0.2^2), \\
    \Omega_b &\sim \mathcal{N}(0.0492,\, 0.006^2), \\
    \sigma_8 &\sim \mathcal{N}(0.831,\, 0.14^2), \\
    h_0      &\sim \mathcal{N}(0.6727,\, 0.063^2), \\
    n_s      &\sim \mathcal{N}(0.9645,\, 0.08^2), \\
    w_0      &\sim \mathcal{N}_{[-2.0,\,-0.3]}(-1.0,\, 0.9^2),
\end{align*}
where $\mathcal{N}_{>0}$ denotes a half-normal (truncated at zero) and
$\mathcal{N}_{[a,b]}$ a normal truncated to $[a, b]$.
The reference (fiducial) cosmology is the truth value of each parameter:
$(\Omega_c, \Omega_b, \sigma_8, h_0, n_s, w_0)
 = (0.2664,\, 0.0492,\, 0.831,\, 0.6727,\, 0.9645,\, -1.0)$.
 We note that in the 6-parameter inference there are visible projection effects on $\Omega_\mathrm{c}$ in most models, induced by the degeneracy between $\Omega_\mathrm{c}$ and $w_0$ (in the bottom-left panel of Fig.~\ref{fig:sbi_lensing_allparams}). Since $w_0$ is only marginally constrained and has an asymmetric prior, it shifts the mean of the posterior for correlated parameters. As $\Omega_\mathrm{c}$ and $\sig$ are highly correlated as well, this also induces a bias on $\sig$. Curiously, both PersLay statistics seem to be most robust against that. This behavior is consistent across multiple random seeds.

\begin{figure}
\centering
\includegraphics[width=0.6\linewidth]{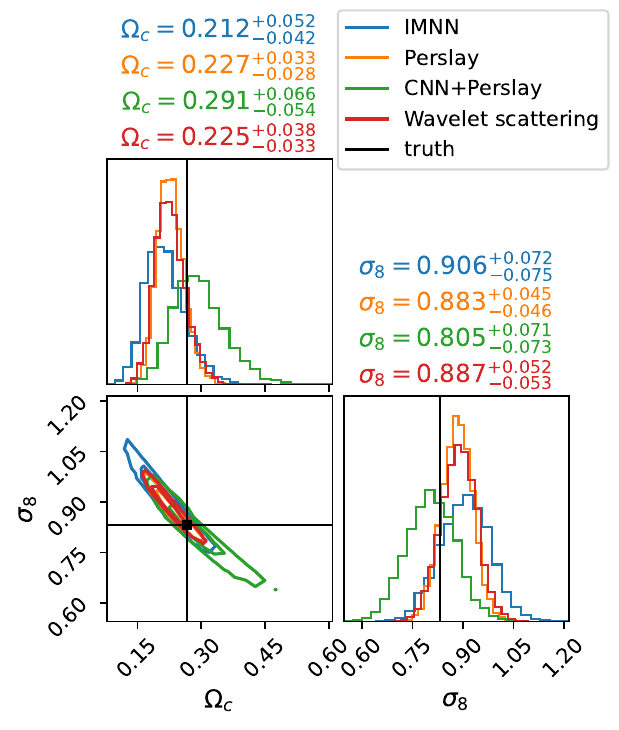}
\caption{Comparison of SBI constraints from different summary statistics. The contours show the 1- and 2-$\sigma$ uncertainties of the different models.}
\label{fig:sbi_lensing_omsig8}
\end{figure}

\begin{figure}
\includegraphics[width=\linewidth]{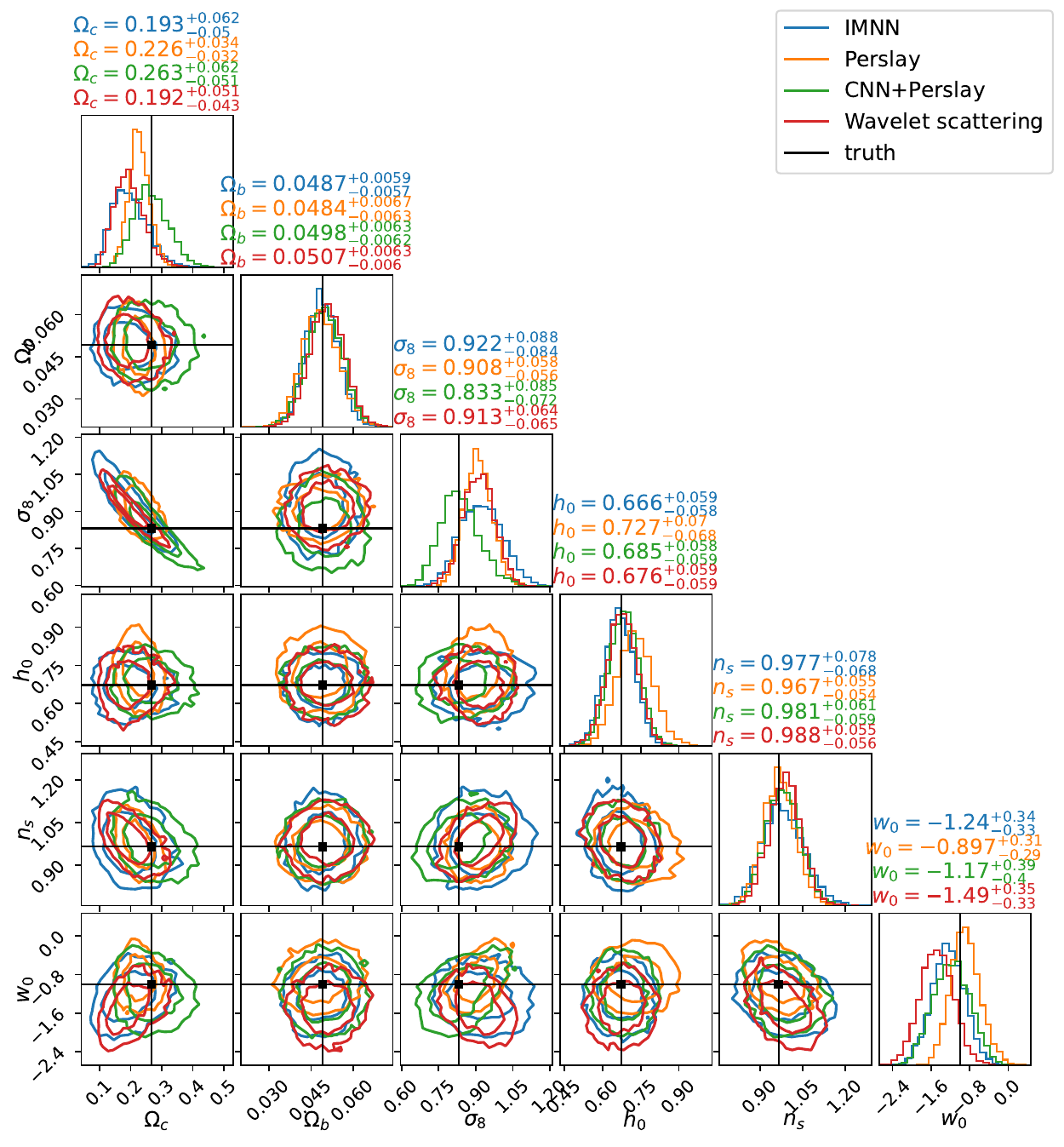}
\caption{Comparison of SBI constraints from different summary statistics over an extended set of cosmological parameters. The contours show the 1- and 2-$\sigma$ uncertainties of the different models. The equation of state of dark energy $w_0$ is particularly interesting to cosmologists.}
\label{fig:sbi_lensing_allparams}
\end{figure}

\section{Architectures and training hyperparameters}
\label{app:architectures}

This appendix describes the architectures, simulation data, compression
schemes, and training hyperparameters used in the reported experiments. The
notation here follows the paper, while Table~\ref{tab:arch_code_map} also gives
the corresponding configuration names used in the released code. The lensing
configuration files are generated over one two-parameter cosmological pair, the \texttt{Om\_s8} pair,
five tomographic bins, and five independent data seeds; the main weak-lensing
tables report  unless stated otherwise.

\subsection{Common pipeline structure and code mapping}
\label{app:arch_common}

Every method is implemented as the composition of Eq. \eqref{eq:pipeline}. For
topological methods, the filtration $f_\phi$ produces a scalar field or filtration values,
persistent homology is computed in dimensions $H_0$ and $H_1$, and
the vectorization $V_\phi$ maps the resulting persistence diagrams to Euclidean features.
For non-topological baselines, the persistence and diagram-vectorization stages
are replaced by the identity.

\begin{table}[ht]
\centering
\scriptsize
\setlength{\tabcolsep}{3.5pt}
\renewcommand{\arraystretch}{1.08}
\caption{Mapping of all the methods used in the analysis. ``Trainable
stage'' refers to parameters updated by the TopoFisher loss. MOPED is fitted
from simulation summaries and carries no optimizer-updated parameters.}
\label{tab:arch_code_map}
\begin{tabular*}{\textwidth}{@{\extracolsep{\fill}}llll@{}}
\toprule
\textbf{Method} & \textbf{Filtration / summary} &
\textbf{Vectorization} & \textbf{Trainable stage} \\
\midrule
$\log(C_\ell)$
& power spectrum
& identity
& none \\
Peak counts
& peak-count histogram
& identity
& none \\
Wavelet scattering
& WST coefficients
& identity
& none \\
Cubical-PI
& fixed cubical PH
& persistence image
& none \\
Cubical-Silhouette
& fixed cubical PH
& persistence silhouette
& none \\
Cubical-Curves
& fixed cubical PH
& differentiable curves
& none \\
TF-Cubical-PersLay
& fixed cubical PH
& PersLay
& vectorization \\
TF-CNN-PersLay
& learned CNN filtration + PH
& PersLay
& filtration and vectorization \\
\IMNN
& strided CNN encoder + dense summary head
& identity
& CNN encoder and dense head \\
\bottomrule
\end{tabular*}
\end{table}

\subsection{Filtrations and fixed scientific summaries}
\label{app:arch_filtrations}

\paragraph{Fixed cubical persistence.}
The fixed topological pipelines use a periodic cubical filtration of the input
field and compute persistence diagrams in homology dimensions $H_0$ and $H_1$.
In the code this is the \texttt{cubical} filtration with
\texttt{homology\_dimensions: [0,1]}, \texttt{periodic: true}, and a GUDHI
backend. The same filtration is used by Cubical-PI, Cubical-Silhouette,
Cubical-Curves, and TF-Cubical-PersLay. For the large lensing runs, persistence
features for some fixed-vectorization baselines are cached on disk when useful
for throughput.

\paragraph{Power spectrum baseline.}
The $\log(C_\ell)$ baseline computes the radially averaged auto-power spectrum
of each tomographic convergence map. We use $30$ logarithmically spaced
multipole bins, exclude the DC mode with $k_{\min}=1$, set the field size to
$512$, and apply a logarithmic transform to the binned power spectrum before
MOPED compression. This baseline is deterministic and has no trainable
parameters.

\paragraph{Peak-count baseline.}
Peak counts are computed by identifying pixels that are strictly greater than
their eight nearest neighbours. We discard a one-pixel boundary, bin peak
heights into $20$ uniformly spaced bins over $\kappa\in[-0.05,0.20]$, and
apply the transform $n_i\mapsto\log(1+n_i)$ to Gaussianize the count vector
before MOPED compression.

\paragraph{Wavelet scattering baseline.}
The wavelet scattering transform (WST) baseline uses the second-order
scattering transform~\citep{mallat2012group} as implemented in
\textsc{Kymatio}~\citep{andreux2020kymatio}. For each $512\times512$ map we
use $J=5$ dyadic scales, $L=4$ orientations, and maximum scattering order $2$.
The feature vector contains the spatial mean $S_0$, $20$ first-order
coefficients $S_1$, and $160$ second-order coefficients $S_2$, for a total of
$181$ features. We apply $\log(1+\cdot)$ to the non-negative $S_1$ and $S_2$
coefficients and leave $S_0$ unchanged.

\paragraph{Learnable CNN filtration for TF-CNN-PersLay.}
The TF-CNN-PersLay pipeline prepends a full-resolution CNN before cubical
persistence: a two-layer CNN with $8$ hidden
channels, $3\times3$ kernels, circular padding, and initialization scale
$0.1$. The CNN output is standardized before persistence, and the persistence
backend computes $H_0$ and $H_1$ diagrams using the \texttt{T} construction with
sub-batches of $100$ maps. The CNN and PersLay stages are both optimized by the
TopoFisher loss. During the main runs, the CNN is
kept frozen for the first $1000$ epochs while PersLay and the MOPED statistics
stabilize, and is then unfrozen for joint optimization.

\paragraph{\IMNN{} baseline.}
The \IMNN{} baseline is an unconstrained non-topological Fisher-neural
reference following the principle of Information Maximising Neural
Networks~\citep{charnock2018automatic}. It consists of a strided convolutional
encoder followed by a dense head that directly outputs
$n_{\rm summaries}=d$ summaries, where $d$ is the number of parameters in the
local Fisher analysis. In all two-parameter experiments reported here,
$d=2$. The persistence and vectorization stages are replaced by the identity,
and the compression stage is also the identity: the dense head itself is the
learned compressor.

For the GRF benchmark, the encoder has channel widths $[16,32,16]$, kernel
sizes $[5,3,3]$, strides $[2,2,2]$, and circular padding. The $64\times64$
input is reduced to an $8\times8$ spatial representation and flattened before
a dense head with hidden width $64$ and output dimension $2$.

For the weak-lensing benchmark, the encoder has channel widths
$[16,32,32,32,16]$, kernel sizes $[7,5,3,3,3]$, strides $[4,2,2,2,2]$, and
circular padding. The $512\times512$ input is reduced to an $8\times8$ spatial
representation and flattened before a dense head with hidden width $128$ and
output dimension $2$.

\subsection{Diagram vectorizations}
\label{app:arch_vectorizations}

\paragraph{Identity.}
The identity vectorization is used by non-topological summaries such as
\IMNN{}, $\log(C_\ell)$, peak counts, and WST: the output of the filtration is directly passed to the compressor.

\paragraph{Persistence images.}
Cubical-PI uses persistence images~\citep{adams2017persistence} with an
$8\times8$ grid, bandwidth $1.0$, and persistence weighting. We compute
features separately for $H_0$ and $H_1$ and concatenate the two outputs.

\paragraph{Persistence silhouettes.}
Cubical-Silhouette uses uniformly weighted persistence silhouettes with
$50$ grid points per homology dimension and a $99$-th percentile support clip.
The $H_0$ and $H_1$ silhouette vectors are concatenated before compression.

\paragraph{Differentiable persistence curves.}
Cubical-Curves uses differentiable cumulative curves of diagram coordinates.
For the lensing runs, the vectorization uses birth and death curves
(\texttt{curves: "B,D"}), $50$ grid points, a $99$-th percentile support clip,
minimum persistence $0$, and soft-sigmoid temperature $2$. For the GRF runs,
where diagram values have a different scale, we use the same curve family with
$50$ grid points and the GRF-specific temperature reported in
Table~\ref{tab:grf_training}.

\paragraph{PersLay.}
PersLay~\citep{carriere2020perslay} is the learnable diagram vectorization used
by TF-Cubical-PersLay and TF-CNN-PersLay. It is applied separately to the $H_0$
and $H_1$ diagrams. The main configuration uses a $16$-dimensional point
embedding, a hidden dimension of $32$, no post-pooling MLP
(\texttt{post\_pool\_dim: 0}), and spectral normalization on the linear layers.

\subsection{Compression and Fisher estimation}
\label{app:arch_compressors}

Unless stated otherwise, all field-based GRF and lensing pipelines use MOPED
compression (Appendix~\ref{app:gaussianity}) as the final linear map before
Fisher estimation, with output dimension $d=2$ matching the parameter
dimension. The covariance inverse entering the MOPED operator is
Hartlap-corrected as described in Appendix~\ref{app:gaussianity}. The
\IMNN{} is the only exception: its dense head outputs a two-dimensional
summary directly, so the compression stage is the identity.

\paragraph{Periodic MOPED refitting.}
For fixed pipelines, MOPED is fitted once from the relevant simulator
dataset and held constant. For pipelines with learnable vectorizations or
filtrations, the feature distribution drifts during training, so the MOPED
operator must be periodically re-estimated. In the main lensing
TF-Cubical-PersLay runs, we refit MOPED every $50$ optimizer epochs. In the
full-resolution TF-CNN-PersLay runs, both MOPED and the PersLay input
normalization are refreshed every $50$ epochs using at most $1{,}000$
samples; this cap avoids repeatedly processing the full $17{,}500$-sample
training split at $512^2$ resolution. These refits update only the linear
compression and normalization statistics; the optimizer-updated parameters
remain those listed in Table~\ref{tab:arch_code_map}.

\paragraph{Transfer-Fisher protocol.}
For transfer-Fisher evaluations on LPT maps, the learned nonlinear stages
(CNNs, PersLay weights, and/or \IMNN{} weights) are loaded from the
corresponding lognormal checkpoint and kept fixed. For MOPED-based methods,
the empirical covariance, finite-difference Jacobians, precision matrix,
and MOPED operator are then re-estimated on the LPT evaluation dataset,
exactly as in any Fisher-information measurement. For the \IMNN{}, there is
no MOPED operator to refit; we evaluate the frozen two-dimensional network
output and re-estimate only the Fisher statistics on LPT maps. Transfer
results therefore test whether the learned feature extractor transfers;
they do not retrain or fine-tune learned filtrations, vectorizations, or
neural summaries on LPT maps.

\subsection{Simulation dataset and finite differences}
\label{app:field_generation}

\paragraph{Noisy spiral datasets.}
For the spiral benchmark, each point cloud contains $N=240$ points, with target
parameters $\theta=(\mu,w)$. We use the finite-difference simulator datasets
described in Section~\ref{sec:swiss_roll}: $n_s=n_d=10{,}000$ samples per
configuration and five independent seeds. All point-cloud pipelines operate on
$k$-nearest-neighbour distance features with $k=100$ for the main table.

\paragraph{GRF datasets.}
For the GRF benchmark, each sample is a zero-mean periodic $64\times64$ field
with power spectrum
$P(k)=A_s(k/\kpiv)^{-B}$. We use
$\theta_{\rm fid}=(A_s,B)=(1,B_0)$ with
$B_0\in\{-2,-1,0,1,2\}$, finite-difference steps
$\Delta\theta=(0.1,0.1)$, and
$n_s=n_d=40{,}000$ simulations per configuration over five seeds. The same
train/validation/test split fractions as above are used.

\paragraph{Lensing dataset.}
Weak-lensing convergence maps are generated with the
\texttt{sbi\_lens} lognormal simulator~\citep{lanusse2023sbi_lens} at
$512\times512$ resolution over a $10^\circ\times10^\circ$ field of view and
five tomographic bins. Maps are smoothed with a Gaussian kernel of width
$\sigma=2'$ before any summary statistic is computed. For every configuration
we use $n_s=20{,}000$ fiducial simulations for covariance estimation and
$n_d=20{,}000$ simulations at each finite-difference point for derivative
estimation. The training/validation/test split fractions are
$0.875/0.0625/0.0625$.

For the LPT transfer experiment we use a separate LPT dataset at the same
resolution, smoothing scale, tomographic-bin definition, and
\texttt{Om\_s8} finite-difference grid. Transfer configurations contain no
training block: checkpoints are supplied at runtime, learned weights are
frozen, and Fisher quantities are evaluated on the LPT dataset.

\subsection{Training protocol and hyperparameters}
\label{app:arch_training}

Training uses structured minibatches containing fiducial simulations for
covariance estimation and paired finite-difference simulations for derivative
estimation. At each optimizer step, the loss is the negative
log-determinant Gaussian Fisher loss of Eq.~\eqref{eq:topofisherloss},
optionally augmented by skewness and excess-kurtosis penalties on the
compressed summary. Validation is performed on a held-out simulator dataset, and
early stopping uses the validation Fisher loss.

Unless stated otherwise, optimization uses Adam with weight decay $10^{-4}$,
ReduceLROnPlateau scheduling with multiplicative factor $0.5$ and minimum
learning rate $10^{-6}$, validation every $10$ epochs, and five independent
data seeds.

\begin{table}[ht]
\centering
\scriptsize
\setlength{\tabcolsep}{3.5pt}
\renewcommand{\arraystretch}{1.08}
\caption{Main lensing training hyperparameters. Fixed baselines have no
optimizer-updated parameters; their entries indicate the estimator and split
configuration used by the common training/evaluation driver. Here
$\lambda_s$ and $\lambda_k$ are the skewness and excess-kurtosis penalties.}
\label{tab:lensing_training}
\begin{tabular*}{\textwidth}{@{\extracolsep{\fill}}lccccccp{0.24\textwidth}@{}}
\toprule
\textbf{Pipeline} & \textbf{Epochs} & \textbf{LR} & \textbf{Batch}
& \textbf{Patience} & \textbf{Clip} & $(\lambda_s,\lambda_k)$
& \textbf{Extra settings} \\
\midrule
Fixed MOPED baselines
& 2000 & $10^{-3}$ & 50 & 100 & 1.0 & $(0.05,0.2)$
& No trainable weights; MOPED fit/evaluation only. \\
TF-Cubical-PersLay
& 2000 & $10^{-3}$ & 50 & 100 & 1.0 & $(0.05,0.2)$
& MOPED refit every $50$ epochs; \texttt{min\_delta}$=10^{-6}$. \\
TF-CNN-PersLay
& 2500 & $3{\times}10^{-4}$ & 32 & 200 & 0.5 & None
& LR warm-up $50$; loss clamp $20$; NaN recovery; CNN frozen for first $1000$ epochs; MOPED and PersLay renormalization every $50$ epochs, capped at $1000$ samples. \\
\IMNN
& 4000 & $5{\times}10^{-4}$ & 100 & 400 & 1.0 & $(0.10,0.00)$
& Cosine LR schedule; identity compression; dense head outputs two summaries directly. \\
\bottomrule
\end{tabular*}
\end{table}

All lensing runs use train/validation/test fractions
$0.875/0.0625/0.0625$. The LR scheduler patience is $20$ for the fixed
baselines, TF-Cubical-PersLay, and \IMNN{}, and $30$ for TF-CNN-PersLay.

\begin{table}[ht]
\centering
\scriptsize
\setlength{\tabcolsep}{3.5pt}
\renewcommand{\arraystretch}{1.08}
\caption{GRF training hyperparameters. These settings are shared across
$B_0\in\{-2,-1,0,1,2\}$ and five seeds.}
\label{tab:grf_training}
\begin{tabular*}{\textwidth}{@{\extracolsep{\fill}}lcccccp{0.28\textwidth}@{}}
\toprule
\textbf{Pipeline family} & \textbf{Epochs} & \textbf{LR} & \textbf{Batch}
& \textbf{Patience} & $(\lambda_s,\lambda_k)$ & \textbf{Extra settings} \\
\midrule
Fixed topological baselines
& 2000 & $10^{-3}$ & 500 & 100 & $(0.05,0.20)$
& MOPED compression; no trainable filtration/vectorization. \\
TF-Cubical-PersLay
& 2000 & $10^{-3}$ & 500 & 100 & $(0.05,0.20)$
& MOPED refit every $50$ epochs. \\
TF-CNN-PersLay
& 2000 & $10^{-3}$ & 500 & 100 & $(0.05,0.20)$
& LR warm-up $50$; loss clamp $20$; NaN recovery; MOPED refit every $50$ epochs. \\
\IMNN
& 4000 & $5{\times}10^{-4}$ & 200 & 400 & $(0.10,0.00)$
& Cosine LR schedule; identity compression; dense head outputs two summaries directly. \\
\bottomrule
\end{tabular*}
\end{table}

The GRF differentiable-curve vectorization uses birth, death, and persistence
curves on $50$ grid points with a $99.9$-th percentile support clip and
temperature $0.1$. The lensing curve vectorization instead uses birth and death
curves, a $99$-th percentile support clip, and temperature $2$, matching the
larger dynamic range of $512^2$ convergence maps.

\paragraph{Hardware and caching.}
GRF runs at $64^2$ resolution are CPU-friendly. Lensing runs involving
full-resolution $512^2$ persistence use a single large-memory GPU for the
CNN/persistence forward and backward passes, with sub-batching inside the
persistence layer. Fixed-vectorization lensing baselines use cached persistence
features when available. End-to-end timing and persistence-backend details are
reported separately in the scalability appendix.

\section{Compute resources and running times}
\label{sec:scalability}

All experiments were run on an anonymized institutional HPC cluster with
CPU-only nodes and single-GPU jobs. CPU nodes used AMD EPYC-class processors
with 256\,GB--1\,TB RAM; GPU jobs used one NVIDIA A100 40\,GB or one NVIDIA
H100 80\,GB accelerator. CPU persistence jobs used up to 32 cores per task
(\texttt{gudhi} parallelism), while GPU jobs requested a single accelerator.
Wall-clock times below are medians measured from scheduler logs. For lensing,
one job means one tomographic bin and one random seed; for GRFs, one job means
one spectral index $B_0$ and one random seed.

\begin{table}[t]
\centering
\caption{Representative wall-clock times for the reported pipelines. Lensing
jobs are one tomographic bin $\times$ one seed; GRF jobs are one $B_0$ value
$\times$ one seed. Times are median scheduler wall-clock times.}
\label{tab:compute_times}
\scriptsize
\setlength{\tabcolsep}{4pt}
\renewcommand{\arraystretch}{1.08}
\begin{tabular*}{\textwidth}{@{\extracolsep{\fill}}llcc@{}}
\toprule
\textbf{Benchmark} & \textbf{Pipeline} & \textbf{Resource request} & \textbf{Median time/job} \\
\midrule
\multirow{8}{*}{Lensing}
& $\log(C_\ell)$ & 4 CPU, 32\,GB & 12 min \\
& Peak counts & 4 CPU, 32\,GB & 22 min \\
& Cubical-Silhouette & 32 CPU, 120\,GB & 40 min \\
& Cubical-Curves & 32 CPU, 120\,GB & 40 min \\
& Cubical-PI & 32 CPU, 120\,GB & 48 min \\
& Wavelet scattering & 16 CPU, 120\,GB & 301 min \\
& TF-Cubical-PersLay & 32 CPU, 384\,GB & 63 min \\
& TF-CNN-PersLay & 1 A100/H100, 200\,GB & 190--270 min \\
& \IMNN{} & 1 H100, 200\,GB & 13 min  \\
\midrule
\multirow{4}{*}{GRF}
& Cubical-PI/Silhouette/Curves & 32 CPU, 120\,GB & 4--6 min \\
& TF-Cubical-PersLay & 32 CPU, 128\,GB & 6 min \\
& TF-CNN-PersLay & 1 A100/H100 & 7--8 h \\
& \IMNN{} & 1 A100/H100 & 7 min\\
\bottomrule
\end{tabular*}
\end{table}

\end{document}